\documentclass[letterpaper, 10 pt, journal, twoside]{ieeetran}

\IEEEoverridecommandlockouts                              

\let\labelindent\relax

\usepackage[numbers]{natbib}
\usepackage[bookmarks=true]{hyperref}
\usepackage{amsmath}
\usepackage{amssymb}
\usepackage{color}
\usepackage{graphicx}
\graphicspath{{images/}}

\usepackage{setspace} 

\usepackage{amsmath}
\usepackage{amsfonts}
\usepackage{amssymb}
\usepackage{amsthm}
\usepackage{bm}
\usepackage{bbm}
\usepackage{mathtools}
\usepackage{enumitem}
\usepackage{thmtools,thm-restate}
\usepackage{algorithm}
\usepackage{algorithmic}
\usepackage{color}
\usepackage{graphicx}
\usepackage{comment}

\def\Rbb{\mathbb{R}}

\def\R{\Rbb}

\def\*{\star}

\newcommand{\tr}[1]{ \mathrm{tr}\left( #1\right)}

\usepackage[dvipsnames]{xcolor}

\renewcommand{\tr}{{T}}

\newcommand{\paral}{{\!/\mkern-5mu/\!}}

\newcommand{\lm}{\bm{\lambda}}
\newcommand{\Lag}{\mathcal{L}}
\newcommand{\Ham}{\mathcal{H}}

\newcommand{\p}{\mathbf{p}}
\newcommand{\q}{\mathbf{q}}
\newcommand{\qd}{{\dot{\q}}}
\newcommand{\qdd}{{\ddot{\q}}}

\newcommand{\vv}{\mathbf{v}}

\newcommand{\x}{\mathbf{x}}
\newcommand{\xd}{{\dot{\x}}}
\newcommand{\xdd}{{\ddot{\x}}}
\newcommand{\y}{\mathbf{y}}

\newcommand{\z}{\mathbf{z}}

\newcommand{\f}{\mathbf{f}}
\newcommand{\h}{\mathbf{h}}
\newcommand{\g}{\mathbf{g}}
\newcommand{\mc}{\mathbf{c}}

\newcommand{\zero}{\mathbf{0}}

\newcommand{\J}{\mathbf{J}}
\newcommand{\Jd}{{\dot{\J}}}

\newcommand{\B}{\mathbf{B}}
\newcommand{\C}{\mathbf{C}}

\newcommand{\G}{\mathbf{G}}

\newcommand{\I}{\mathbf{I}}

\newcommand{\M}{\mathbf{M}}

\newcommand{\mP}{\mathbf{P}}
\newcommand{\mR}{\mathbf{R}}

\newcommand{\wt}[1]{{\widetilde{#1}}}
\newcommand{\wh}[1]{{\widehat{#1}}}

\usepackage{amsmath}  
\usepackage{amssymb}
\usepackage{accents}  
\usepackage{amsthm}

\theoremstyle{plain}
\newtheorem{theorem}{Theorem}[section]
\newtheorem{lemma}[theorem]{Lemma}

\theoremstyle{definition}
\newtheorem{definition}[theorem]{Definition}

\theoremstyle{remark}

\makeatletter
\let\save@mathaccent\mathaccent
\newcommand*\if@single[3]{%
  \setbox0\hbox{${\mathaccent"0362{#1}}^H$}%
  \setbox2\hbox{${\mathaccent"0362{\kern0pt#1}}^H$}%
  \ifdim\ht0=\ht2 #3\else #2\fi
  }
\newcommand*\rel@kern[1]{\kern#1\dimexpr\macc@kerna}
\newcommand*\widebar[1]{\@ifnextchar^{{\wide@bar{#1}{0}}}{\wide@bar{#1}{1}}}
\newcommand*\wide@bar[2]{\if@single{#1}{\wide@bar@{#1}{#2}{1}}{\wide@bar@{#1}{#2}{2}}}
\newcommand*\wide@bar@[3]{%
  \begingroup
  \def\mathaccent##1##2{%
    \let\mathaccent\save@mathaccent
    \if#32 \let\macc@nucleus\first@char \fi
    \setbox\z@\hbox{$\macc@style{\macc@nucleus}_{}$}%
    \setbox\tw@\hbox{$\macc@style{\macc@nucleus}{}_{}$}%
    \dimen@\wd\tw@
    \advance\dimen@-\wd\z@
    \divide\dimen@ 3
    \@tempdima\wd\tw@
    \advance\@tempdima-\scriptspace
    \divide\@tempdima 10
    \advance\dimen@-\@tempdima
    \ifdim\dimen@>\z@ \dimen@0pt\fi
    \rel@kern{0.6}\kern-\dimen@
    \if#31
      \overline{\rel@kern{-0.6}\kern\dimen@\macc@nucleus\rel@kern{0.4}\kern\dimen@}%
      \advance\dimen@0.4\dimexpr\macc@kerna
      \let\final@kern#2%
      \ifdim\dimen@<\z@ \let\final@kern1\fi
      \if\final@kern1 \kern-\dimen@\fi
    \else
      \overline{\rel@kern{-0.6}\kern\dimen@#1}%
    \fi
  }%
  \macc@depth\@ne
  \let\math@bgroup\@empty \let\math@egroup\macc@set@skewchar
  \mathsurround\z@ \frozen@everymath{\mathgroup\macc@group\relax}%
  \macc@set@skewchar\relax
  \let\mathaccentV\macc@nested@a
  \if#31
    \macc@nested@a\relax111{#1}%
  \else
    \def\gobble@till@marker##1\endmarker{}%
    \futurelet\first@char\gobble@till@marker#1\endmarker
    \ifcat\noexpand\first@char A\else
      \def\first@char{}%
    \fi
    \macc@nested@a\relax111{\first@char}%
  \fi
  \endgroup
}
\makeatother

\newcommand{\fullversiononly}{\iftrue}
\newcommand{\compressedversiononly}{\iffalse}

\fullversiononly
\newcommand{\apx}{Appendix}
\else
\newcommand{\apx}{\cite{Wyk2021GeometriFabricsPhysicsArXivAbbrev} Appendix}
\fi

\title{Geometric Fabrics: Generalizing Classical Mechanics to Capture the Physics of Behavior
}
\fullversiononly
\author{Karl Van Wyk$^{1*}$, Mandy Xie$^{1,2*}$, Anqi Li$^{1,3}$, Muhammad Asif Rana$^{1,2}$,
Buck Babich$^{1}$, Bryan Peele$^{1}$, Qian Wan$^{1}$, \\
Iretiayo Akinola$^{1}$, Balakumar Sundaralingam$^{1}$, 
Dieter Fox$^{1,3}$, Byron Boots$^{1,3}$, and Nathan D. Ratliff$^{1}$ \\
{\small
*Equal contrib,
$^{1}$NVIDIA, 
$^{2}$Georgia Tech, 
$^{3}$University of Washington.}
}
\else 
\author{Karl Van Wyk$^{1*}$, Mandy Xie$^{1,2*}$, Anqi Li$^{1,3}$, Muhammad Asif Rana$^{1,2}$,
Buck Babich$^{1}$, Bryan Peele$^{1}$, Qian Wan$^{1}$, \\
Iretiayo Akinola$^{1}$, Balakumar Sundaralingam$^{1}$, 
Dieter Fox$^{1,3}$, Byron Boots$^{1,3}$, and Nathan D. Ratliff$^{1}$
\thanks{Manuscript received: September, 9, 2021; Revised December, 1, 2021; Accepted December, 27, 2021.}
\thanks{This paper was recommended for publication by Editor Lucia Pallottino upon evaluation of the Associate Editor and Reviewers' comments.}
\thanks{$^{1}$Karl Van Wyk, Buck Babich, Bryan Peele, Qian Wan, Iretiayo Akinola, Balakumar Sundaralingam, Dieter Fox, Byron Boots, and Nathan Ratliff are with Nvidia, USA
        {\tt\footnotesize kvanwyk@nvidia.com}}
\thanks{$^{2}$ Mandy Xie and Muhammad Asif Rana are with the School of Computing, Georgia Institute of Technology, USA
        {\tt\footnotesize manxie@gatech.edu}}
        
\thanks{$^{3}$ Anqi Li, Dieter Fox, and Byron Boots are with the School of Computer Science and Engineering, University of Washington, USA 
        {\tt\footnotesize anqil4@cs.washington.edu}}

\thanks{Digital Object Identifier (DOI): see top of this page.}
}
\fi

\begin{document}

\maketitle

\begin{abstract}
Classical mechanical systems are central to controller design in energy shaping methods of geometric control.
However, their expressivity is limited by position-only metrics and the intimate link between metric and geometry.
Recent work on Riemannian Motion Policies (RMPs) has shown that shedding these restrictions results in powerful design tools, but at the expense of theoretical stability guarantees. In this work, we generalize classical mechanics to what we call geometric fabrics, whose expressivity and theory enable the design of systems that outperform RMPs in practice. Geometric fabrics strictly generalize classical mechanics forming a new physics of behavior by first generalizing them to Finsler geometries and then explicitly bending them to shape their behavior while maintaining stability.
We develop the theory of fabrics and present both a collection of controlled experiments examining their theoretical properties and a set of robot system experiments showing improved performance over a well-engineered and hardened implementation of RMPs, our current state-of-the-art in controller design.

\end{abstract}

\section{Introduction}

\IEEEPARstart{C}{lassical} mechanical systems \cite{ClassicalMechanicsTaylor05} are the workhorse of controller design in geometric control\footnote{Geometric control encompasses operational space control \cite{khatib1987unified} and potential field methods \cite{KhatibPotentialFields1985}; we refer to them all here as geometric control.} \cite{bullo2004geometric}. 
Their popularity derives from both powerful conservation properties \cite{LeonardSusskindClassicalMechanics} lending to simply Lyapunov stability and convergence analysis \cite{khalil2002nonlinear} and their link to Riemannian geometry (see geometric mechanics \cite{bullo2004geometric,ratliff2021geometry}), enabling the shaping of, for instance, straight-line end-effector behavior. However, classical mechanical systems have fundamentally limited expressivity.

When designing behavior in parts---such as combining target reaching with obstacle avoidance, joint limit avoidance, posturing, and speed regulation---these classical mechanical systems combine as metric-weighted averages of component policies (see Section~\ref{sec:ShortcomingsOfClassicalMechanics}). Three limitations are clear from this perspective: 
(1) priority metrics depend only on position and not full system state; (2) the geometry is fundamentally linked to the metric making it impossible to design both independently; and (3) controllers, as a result, overly rely on potentials, dampers and unnatural task hierarchies.
Recent work in Riemannian Motion Policies (RMPs) \cite{cheng2018rmpflow}, 
demonstrated that removing the restrictions of classical mechanics can result in powerful design tools. However, RMPs sacrifice theoritcal insight for expressivity, forcing practitioners to gain experience before becoming proficient. 

{\em Geometric fabrics}, a new physics of behavior, generalize classical mechanics while retaining the powerful mathematics of classical systems and promoting flexibility to circumvent the above issues. Geometric fabrics are what we call {\em bent} Finsler geometries \cite{ratliff2021geometry}---the Finsler system generalizes the Riemannian geometry of classical systems giving us velocity-dependent metrics, and the bending terms enable the independent shaping of associated policies. 
Constraint forces are simple examples of bending terms commonly deployed to follow trajectories or trace surfaces \cite{Peters_AR_2008}, but our bending terms extend far beyond constraints. We derive fabrics that bend Finsler geometries to match any given nonlinear geometry using a technique called {\em geometry energization}.

One may view fabrics as a class of systems that instantaneously look like classical systems (with mass matrices, corresponding inertias, and forces), but which change from moment to moment as a function of state, similar to variable impedance systems; even concepts such as contact models generalize. In most cases, techniques for classical systems apply to fabrics, but now with the expressivity of RMPs.

Concretely, this work: (1) generalizes classical mechanics to capture the flexibility required for behavioral design; (2) develops an expressive class of RMPs that are stable, path consistent, and agnostic to parameterization (covariant); (3) develops a fabric component algebra analogous to the RMP algebra; and (4) presents experiments validating the theory and showing we can exploit the theoretical properties to out-perform a well-tuned implementation of RMPflow.

\fullversiononly
We start by reviewing background in Section~\ref{sec:Background} and discuss the limitations of classical systems. 
We then define geometric fabrics in Section~\ref{sec:GeometricFabrics}, show that constrained fabrics are themselves fabrics, and present a simple stability and convergence theorem. 
Subsection~\ref{sec:EnergizingHd2Geometries} presents energization, showing that fabrics are in a sense complete and can take on the shape of any geometry. Finally, we discuss computation in generalized coordinates in Subsection~\ref{sec:GeneralizedCoordinates} and use those results to define an algebra of {\em fabric components} in Section~\ref{sec:DesigningFabrics} which can be used to derive computational algorithms.
Section~\ref{sec:ControlledParticleExperiments} then analyzes empirically the advantages of geometric fabrics over standard unbent Riemannian and Finsler systems, and Section~\ref{sec:SystemExperiments} studies a full robot system implementation of fabrics, comparing it to a well-engineered and hardened implementation of RMPs.

The appendix and our website\footnote{\url{https://sites.google.com/nvidia.com/geometric-fabrics-behavior}} contains extended discussion, proofs, and additional experimental demonstrations.
\fi


\section{Related Work}
\label{sec:RelatedWork}
Second-order systems are natural for control since robots, themselves, are second-order systems (see \cite{IjspeertDMPs2013} for similar arguments), so we restrict our attention here to those.

Geometric control \cite{bullo2004geometric,khatib1987unified,Peters_AR_2008}, which uses energy shaping of classical mechanical systems for behavioral design, is the most directly related area.
We discuss the limitations of classical systems in Section~\ref{sec:ShortcomingsOfClassicalMechanics} and propose geometric fabrics to eliminate them. Energy shaping methods can leverage this broader class of systems for more expressive design.
Dynamic Movement Primitives (DMPs) \cite{IjspeertDMPs2013} are another closely related area. DMPs operate by perturbing stable systems with a forcing term that either vanishes or is periodic---the resulting system inherits its stability from the underlying stable system. This idea is again highly-compatible with geometric fabrics since we can swap out the underlying stable system with any stable second-order system, including geometric fabrics (see \apx~\ref{sec:GeneralizingDmps} for details). 
Optimal control \cite{OptimalControlEstimationStengel94} is another relevant area and again highly compatible. Any optimal controller optimizes over a dynamics function. Traditionally, that dynamics function is derived from classical mechanics, but fabrics are a viable generalized alternative that would enable building behavior into the dynamics to aid the optimization. 

An alternative approach to fabrics, developed independent of this work, is the Pullback Bundle Dynamical System (PBDS), derived as a class of stable Riemannian systems on the tangent bundle \cite{BylardBonalliEtAl2021}. 
PBDSs separate priority metric and policy design but in turn require task policies to be srictly Riemannian, a subset of the HD2 geometries of fabrics. Contact dynamics is also less clear, although such extensions are likely possible.
Another related area is Control Lyapunov Functions (CLF) \cite{Li2019LyapunovRmps} which constitute a general technique for stabilizing systems. CLFs work by projecting unstable systems onto a Lyapunov stable class introducing potentially large projection errors. With fabrics we directly design inherently stable nonlinear systems so the final system both captures the desired behavior and is stable by construction.
Finally, earlier work introduced Geometric Dynamical Systems (GDS) \cite{cheng2018rmpflow} as a stable class of RMPs. We now know that GDSs are better expressed as Finsler systems which are unbent fabrics. \apx~\ref{apx:Gds} details this relation and Subsection~\ref{sec:ControlledParticleExperiments} shows empirically their limitations.

Other control approaches, including $\lambda_0$-PMP models \cite{tommasino2017task}, are first-order differential IK models related to classical Riemannian systems but lack the second-order fidelity and velocity dependence we address here. Fractal impedance methods \cite{tiseo2021passive} define locally optimal controllers that asymptotically track simultaneous position and force profiles. However, these are specialized controllers unlike geometric fabrics which describe a very general class of second order dynamical systems without any explicit notion of trajectory tracking or specific objectives. Moreover, composable energy policies \cite{urain2021composable} optimize actions over a product of stochastic policies to resolve multiple policy conflicts. Geometric fabrics address this issue as well through geometric policies. Alternative works try to further resolve multi-policy conflicts with hierarchical prioritization and null-space projections \cite{sentis2006whole, khatib1987unified,dietrich2018HierarchicalOpSpaceControl}. However, these structures are overly strict and relative priorities should instead continuously adjust.

\fullversiononly
Finally, Finsler and spray geometry in mathematics \cite{bucataru2007FinslerLagrangeDynamics} are, of course, highly relevant. That literature, however, remains insurmountably challenging for most roboticists. We work from the re-derivations presented in \cite{ratliff2021geometry}; it's an area of future work to lower the barrier-to-entry and leverage related ideas from the literature. To the best of our knowledge geometric fabrics and results on constrained fabrics and energization are novel, as are the fabric component algebra and resulting tools for putting them into practice.
\fi

\section{Background}
\label{sec:Background}
\subsection{Finsler and general nonlinear geometries}
\label{sec:FinslerAndHd2Geometries}

We briefly review some of the main ideas behind Finsler and general nonlinear geometries of paths (Homogeneous of Degree 2 (HD2) geometries)---see \cite{ratliff2021geometry} for details. We call any smooth differential equation $\qdd + \h(\q, \qd) = \zero$ an HD2 geometry when $\h$ is HD2 in velocities so that $\h(\q, \alpha\qd) = \alpha^2\h(\q, \qd)$. We often write it in resolved form $\qdd = -\h(\q, \qd) = \pi(\q, \qd)$ and refer to $\pi(\q, \qd)$ as a {\em geometric policy} to give policy semantics. The HD2 property enforces that integral curves are \textit{path consistent}, i.e., they follow geometric speed-independent paths \cite{ratliff2021geometry}. 
{\em Finsler geometries} are HD2 geometries defined by their minimum length geodesics using the calculus of variations. Their generalized arc-length measures are defined by a Finsler structure $\Lag_g(\q,\qd)$:

\begin{definition} \label{def:FinslerStructure}
A Finsler structure is a Lagrangian with the following three properties:
\begin{enumerate}
    \item 
    $\Lag_g(\q,\qd)\geq 0$ 
    with equality if and only if $\qd=\zero$.
    \item $\Lag_g$ is positive homogeneous (HD1) in $\qd$ so that $\Lag_g(\q,\alpha\qd) = \alpha\Lag_g(\q, \qd)$ for $\alpha \geq 0$.
    \item $\partial^2_{\qd\qd}\Lag_e\succeq 0$ when $\qd\neq\zero$ where $\Lag_e = \frac{1}{2}\Lag_g^2$.
\end{enumerate}
\end{definition}
The first two properties enforce that $\Lag_g(\q, \qd)$ is a speed-independent measure of length (see \cite{ratliff2021geometry}). The third property ensures the equations of motion of the associated {\em Finsler energy} $\Lag_e$, derived from the Euler-Lagrange equation as $\frac{d}{dt}\partial_\qd\Lag_e - \partial_\q\Lag_e = \M(\q,\qd)\qdd + \bm{\xi}(\q, \qd) = \zero$, with $\M = \partial_{\qd\qd}^2\Lag_e$ (mass / metric) and $\bm{\xi} = \partial_{\qd\q}\Lag_e\qd - \partial_\q\Lag_e$ (geometric terms), are well-formed with positive definite mass. These (closed) equations of motion (absent potential, damping, or external forces) define an HD2 (Finsler) geometry \cite{ratliff2021geometry}.

A simple example of a Finsler geometry is Riemannian geometry. The Riemannian Finsler structure measures length $\Lag_g(\q, \qd) = \big(\qd^T\M(\q)\qd\big)^{\frac{1}{2}} = \|\qd\|_\M$ using a symmetric positive definite metric $\M(\q)$ dependent on position only. The corresponding energy is the well-known classical kinetic energy $\Lag_e = \frac{1}{2}\Lag_g^2 = \frac{1}{2}\|\qd\|_\M^2 = \frac{1}{2}\qd^T\M(\q)\qd$ for mass $\M$. This link between geometry and system energy is reflected more generally in the link between Finsler geometry and Finsler energy, with equations of motion tracing geodesics.

Every Lagrangian system has an associated conserved energy (its Hamiltonian) $\Ham = \partial_\qd\Lag^\tr\qd - \Lag = \p^\tr\qd - \Lag$ where $\p = \partial_\qd\Lag$ is the {\em generalized momentum}. When the Lagrangian is a Finsler energy $\Lag_e = \frac{1}{2}\Lag_g^2$ with equations of motion $\M(\q, \qd)\qdd + \bm{\xi}(\q,\qd) = \zero$, the Hamiltonian can be expressed $\Ham_e = \Lag_e = \frac{1}{2}\qd^\tr\M\qd$, showing that $\Lag_e$ itself is the energy and can be expressed in the classical form \cite{ratliff2021geometry}.

The Finsler metric $\M(\q, \qd)$ is Homogeneous of Degree 0 (HD0) which means $\M(\q, \alpha\qd) = \M(\q, \qd)$. $\M$ is, therefore, dependent on  velocity direction $\widehat{\qd}$ but not on speed. That means, for any system $\M(\q,\qd)\qdd + \f(\q,\qd) = \zero$ (where $\f$ may capture more than just $\bm{\xi}$), the resolved equation $\qdd = -\M^{-1}\f$ defines a geometric policy if and only if $\f$ is HD2.

\subsection{Energy Shaping with Classical Mechanical Systems}
\label{sec:ShortcomingsOfClassicalMechanics}

Classical mechanical systems have traditionally been important tools for energy shaping in controller design \cite{bullo2004geometric}. 
The main idea of energy shaping is to design a virtual classical mechanical system that captures a desired behavior, then control the physical system to behave like it thereby inheriting the stability of the virtual system. This idea is powerful which is why it has been so important and popular in control. But we show here that classical systems have limited expressivity as tools for behavioral design.

Let $\mathcal{Q}$ be a generalized coordinate domain constrained by $\C(\q)=\zero$ with Jacobian $\J(\q)$. 
A constrained classical mechanical system evolves according to the equation \cite{ClassicalMechanicsTaylor05}:
\begin{align} \nonumber
    \M(\q)\qdd + \bm{\xi}(\q, \qd) + \J(\q)^\tr\lm(\q,\qd) = -\partial_\q\psi(\q) - \B(\q,\qd)\qd,
\end{align}
where $\M(\q)\qdd + \bm{\xi}(\q, \qd) = \zero$ form the equations of motion of an unforced system with kinetic energy $\mathcal{K}(\q, \qd) = \frac{1}{2}\qd^\tr\M(\q)\qd$ (the term $\bm{\xi}(\q,\qd)$ captures fictitious forces which we call here {\em geometric forces} in reference to their role in Finsler geometry), the function $\psi(\q)$ is a potential function, $\B(\q,\qd)$ is a positive definite damping matrix, and $\lm(\q,\qd)$ are the Lagrange multipliers of the constraint force. With $\B=\zero$ the system conserves total energy $\mathcal{E}(\q,\qd) = \mathcal{K}(\q,\qd) + \psi(\q)$; with nonzero $\B$ the energy dissipates, and the system converges to a local minimum of the potential on the constrained domain and is therefore inherently stable.

Unfortunately, a simple analysis shows that these systems have limited expressivity. Consider multiple classical systems $\M_i(\q)\qdd + \f_i(\q,\qd) = \bm{\tau}_i$, where now $\f_i$ captures all forces excluding constraints (fictitious, potential, and damping---constraint forces are often added to the combined system after summing, so we ignore them here) and we single out $\bm{\tau}_i$ as the force the other systems place on system $i$. For instance, these component systems may express reaching a target, avoiding obstacles, joint limit avoidance, and posturing of the arm. The internal forces $\bm{\tau}_i$ must sum to zero $\sum_i\bm{\tau}_i = \zero$, so we have $\sum_i\big(\M_i\qdd + \f_i\big) = \sum_i\bm{\tau}_i = \zero$. Assuming each $\M_i$ is full rank, we can write $\pi_i = -\M_i^{-1}\f_i$ and express the resulting combined acceleration 
\fullversiononly
as\footnote{A similar result holds if $\M_i$ is not full rank. Intuitively, we can add a small $\epsilon$ to the diagonal then take the limit $\epsilon\rightarrow 0$ in Equation~\ref{eqn:ClassicalMetricWeightedAverage}.}
\else
as
\fi
\begin{align} \label{eqn:ClassicalMetricWeightedAverage}
    \qdd = \left(\sum_i\M_i\right)^{-1}\sum_i\M_i\pi_i,
\end{align}
which is simply a metric-weighted average of constituent policies $\pi_i$, where each policy takes the form:
\begin{align} \nonumber
    \pi_i(\q, \qd) = -\M_i^{-1}\Big(\bm{\xi}_i(\q,\qd) + \partial_\q\psi_i(\q) + \B_i(\q,\qd)\qd\Big).
\end{align}
Each component policy is defined by geometric forces, its potential function, and its damping matrix; they averaged together using the mass matrices as spectral priority weights (with Eigenvalues expressing different priorities along each Eigenvector). Three limitations manifest from this analysis.

First, priority metrics are functions of position only and not full position-velocity state. Many tasks are best expressed with velocity aware priorities which classical metrics cannot represent. E.g. any barrier avoidance task (obstacle, self, joint limit avoidance) should care about the barrier while moving toward it, but should forget about it when moving away.

Second, geometric forces can be powerful---they define the Riemannian geometry of the system associated with the metric $\M(\q)$ \cite{bullo2004geometric,ratliff2021geometry}---but in classical systems they are intrinsically linked to the metric which also plays the role of the policy's priority matrix. Only one of the two can be designed, the other becomes an artifact.
This coupling results in fighting between components (when the metric is compromised) or forces an over-reliance on potentials and dampers (when the geometry is compromised). An example of a Riemannian geometry working against the policy is the robot's own (unshaped) inertia which will fight a potential.
Likewise, straight-line end-effector motion is a useful shaped geometry, but that defines a largely irrelevant priority metric.

Third, when the geometric forces cannot be used, we must shape the behavior entirely using potential functions and dampers. Potential functions add to the total potential energy affecting where local minima fall resulting in task conflict. 
Moreover, potentials are only position-dependent and can induce spring-like oscillations when damping is insufficient. Alternatively, one can increase damping, or even entirely shape the behavior using dampers, but that can slow the system when mutliple contributions combine. Geometric forces, on the other hand, are path consistent and will shape the behavior without shifting the system local minima.

\fullversiononly
Note that the second and third limitations are largely the motivation for hierarchical task prioritization \cite{dietrich2018HierarchicalOpSpaceControl}---when multiple potentials may conflict or priority metrics are compromised, designers preserve task fidelity by instituting a task hierarchy whereby lower-priority tasks are projected into the null-space of higher-priority tasks \cite{sentis2006whole, khatib1987unified}. However, these task hierarchies are often pragmatic solutions to these fundamental problems and somewhat artificial in practice. For instance, the relative priorities of reaching targets and avoiding obstacles should adjust continuously; it is unclear which is globally higher priority.
\fi

Geometric fabrics remove these limitations by (1) introducing velocity-dependent metrics using Finsler geometries, and (2) decoupling system geometry from priority metric using bending terms. Now, the system's behavior is primarily directed by the fabric, the potential is applied for a specific task, and damping for convergence and speed control.

\section{Geometric Fabrics: A New Physics of Behavior}

Geometric fabrics are a strict generalization of classical mechanics (with classical systems being a type of fabric) which retain the powerful mathematical properties exploited in control while removing the expressivity limitations outlined in Subsection~\ref{sec:ShortcomingsOfClassicalMechanics}. In particular, fabrics' priority metrics depend on full state and are decoupled from their associated geometric policy. Stability of these systems is easily analyzed using energy conservation (see Theorem~\ref{thm:GeometricFabricStability}). We derive fabrics as bent Finsler geometries using what we call {\em bending terms} which generalize constraint forces. These bending terms are powerful and can be used to bend a Finsler geometry to align with any desired HD2 geometry using a process we call geometry {\em energization}. We additionally derive formulas for calculation in generalized coordinates which we use in Section~\ref{sec:DesigningFabrics} to construct an algebra of fabric components to aid systems engineering.

\subsection{Bending Finsler geometries}
\label{sec:BendingFinslerGeometries}

An important concept in geometric fabrics is the {\em bending} of a Finsler geometry. Geometric bending is implemented using system modification terms that both perform no work (conserve energy) and are geometric (HD2).
An example of such a bending term is the classical constraint force that keeps a system moving along a constraint surface. These forces are zero-work by design (see the Principal of Virtual Work \cite{ClassicalMechanicsTaylor05}) and are known to preserve the geometric character of the system. Bending terms, as defined and characterized in the following Lemma, generalize classical constraints.

\begin{lemma} \label{lma:EnergyConservation}
Let $\Lag_e$ be a Finsler energy with unforced equations of motion $\M(\q,\qd)\qdd + \bm{\xi}(\q,\qd) = \zero$, and let $\psi(\q)$ be a potential function and $\B(\q, \qd)$ be a positive definite damping matrix. The modified system $\M\qdd + \bm{\xi} + \f_f = -\partial\psi - \B\qd$ with modifying term $\f_f(\q,\qd)$ has total energy $\Ham = \Lag_e + \psi$ that decreases at a rate
\begin{align}
    \dot{\Ham} = -\qd^\tr\B\qd
\end{align}
if and only if $\qd^\tr\f_f = \zero$ everywhere. We call such a term $\f_f$ a {\em zero-work modification}. When $\f_f$ is additionally HD2, the unforced system (with $\psi=0$ and $\B=\zero$) is geometric and we say $\f_f$ is a {\em bending term} that bends the Finsler geometry. Moreover, if $\psi$ is lower bounded, $\Ham$ is lower bounded. 
\end{lemma}
\begin{proof}
See \apx~\ref{apx:LemmaEnergyConservation}.
\end{proof}
Bending terms are geometric forces that are everywhere orthogonal to the direction of motion and thereby conserve energy by performing no work. 

\subsection{Geometric fabrics}
\label{sec:GeometricFabrics}

Geometric fabrics generalize classical systems in two steps: (1) they generalize the classical Riemannian geometries to the broader class of Finsler geometries; and (2) they bend the resulting Finsler geometry using bending terms which include, but are not limited to, constraint forces.

\begin{definition}[Geometric fabric]
Let $\Lag_e = \frac{1}{2}\Lag_g^2$ be a Finsler energy derived from Finsler structure $\Lag_g$
with equations of motion $\M(\q, \qd)\qdd + \bm{\xi}(\q, \qd) = \zero$. 
A geometric fabric is a system that evolves according to
\begin{align} \label{eqn:GeometricFabricEquationsOfMotion}
    &\M(\q, \qd)\qdd + \bm{\xi}(\q, \qd) + \f_f(\q, \qd) = \zero,
\end{align}
where $\f_f(\q, \qd)$ is any bending term. Geometric fabrics are denoted $\mathcal{F} = (\Lag_e, \f_f)$.
\end{definition}
The fabric defines a nominal behavior independent of a specific task. Potential functions then push away from that nominal behavior to drive the system toward task goals, and damping dissipation ensures convergence. The full equations of motion of the system then 
\fullversiononly
become:\footnote{A broader theory of {\em optimization} fabrics \cite{optimizationFabricsForBehavioralDesignArXiv2020} drops the geometric requirement but retains the optimization properties outlined in Theorem~\ref{thm:GeometricFabricStability}. Geometric fabrics are a type of fabric that most naturally generalizes classical mechanics. We refer to geometric fabrics simply as fabrics.}
\else 
become:
\fi
\begin{align}\label{eqn:ForcedFabricEquationsOfMotion}
    &\M(\q, \qd)\qdd + \bm{\xi}(\q, \qd) + \f_f(\q, \qd) + \J(\q)^\tr\lm(\q,\qd) \\\nonumber
    &\ \ \ \ = -\partial_\q\psi(\q) - \B(\q, \qd)\qd,
\end{align}
where $\psi(\q)$ is a potential function, $\B(\q, \qd)$ is a positive definite damping matrix, and $\lm$ are the Lagrange multipliers of constraint forces with constraint Jacobian $\J$. We see many similarities to the classical equations given in Subsection~\ref{sec:ShortcomingsOfClassicalMechanics}. The potential function injects energy into the system and drives it toward its local minimum, while the damper bleeds energy off for stability. Lemma~\ref{lma:ConstrainedGeometricFabrics} characterizes constrained geometric fabrics and shows that a constrained geometric fabric is itself a geometric fabric. Theorem~\ref{thm:GeometricFabricStability} proves stability of these systems.

\begin{lemma} \label{lma:ConstrainedGeometricFabrics}
Let $(\Lag_e, \f_f)$ be a geometric fabric with Finsler energy $\Lag_e$, bending term $\f_f$, and associated equations of motion $\M\qdd + \bm{\xi} + \f_f = -\partial\psi - \B\qd$ for potential $\psi$ and damping matrix $\B$. Under equality constraint $\C(\q) = \zero$, assuming $(\q,\qd)$ already satisfy the constraint, the resulting constrained equations of motion are
\begin{align}\label{eqn:ConstrainedGeometricFabrics}
    \M\qdd + \mP_\paral\big(\bm{\xi}+\f_f\big) + \mP_\perp \M\Jd\qd
    = -\mP_\paral\big(\partial\psi + \B\qd\big)
\end{align}
where, $\J = \partial_\q\C$ denotes the constraint Jacobian, $\mP_\perp = \J^\tr(\J\M^{-1}\J^\tr)^{-1}\J\M^{-1}$ projects orthogonally to the constraints, and $\mP_\paral = \I - \mP_\perp$ projects parallel to the constraints. Importantly, the left hand side defines a fabric $(\Lag_e, \wt{\f}_f)$ where $\wt{\f}_f = \mP_\perp(\M\Jd\qd - \bm{\xi}) + \mP_\paral\f_f$ is a {\em constrained} fabric.
\end{lemma}
\begin{proof}
See \apx~\ref{apx:LemmaConstrainedGeometricFabrics}.
\end{proof}

The constrained system in Equation~\ref{eqn:ConstrainedGeometricFabrics} is intuitive. $\mP_\paral\big(\bm{\xi}+\f_f\big)$ projects the original fabric to operate parallel to the constraint and $-\mP_\paral(\partial\psi+\B\qd)$ does the same with the forcing terms. $\mP_\perp\M\Jd\qd$ independently creates the required curvature force to keep the system moving along the constraint, modifying the original fabric's bending term. 

\subsection{Stability analysis}
\label{sec:StabilityAnalysis}

The following theorem presents a general result on the stability and convergence of forced geometric fabrics.
\begin{theorem}[Geometric fabric stability] \label{thm:GeometricFabricStability}
Suppose $\mathcal{F} = (\Lag_e, \f_f)$ is geometric fabric operating under equality constraints $\mc(\q)=\zero$ and inequality constraints $\g(\q)\geq\zero$ (both of which may be empty). Then for any $\psi(\q)$ lower-bounded on the constrained domain and any bounded positive-definite damping matrix $\B(\q, \qd)$ the system evolving according to Equation~\ref{eqn:ForcedFabricEquationsOfMotion} (with $\C$ consisting of $\mc$ and the active subset of $\g$)
converges to a local minimum of the optimization problem
\begin{align}
    &\min_{\q\in\mathcal{Q}} \psi(\q)
    \ \ \mathrm{s.t.}\ \ \mc(\q) = \zero\ \ \mathrm{and}\ \ \g(\q)\geq\zero
\end{align}
under any contact model contributing as equality constraints with dissipation when in contact and conserving or dissipating energy on impact.  
\end{theorem}
\begin{proof}
See \apx~\ref{apx:TheoremGeometricFabricStability}
\end{proof}

\subsection{Completeness: Energizing HD2 geometries}
\label{sec:EnergizingHd2Geometries}

The next theorem is central to the theory of fabrics. It shows fabrics are complete in the sense that they can be bent to align with any given geometry. Completeness shows we have independent control over priority metric design (Finsler metric) and geometric policy design (desired geometry).
\begin{theorem}[Energized fabrics] \label{thm:SystemEnergization}
Let $\qdd = \pi(\q, \qd)$ be an HD2 geometry, and suppose $\Lag_e$ is a Finsler energy with equations of motion $\M\qdd + \bm{\xi} = \zero$ and energy $\Ham_e=\Lag_e$. Then $\qdd = \pi(\q, \qd) + \alpha_{\Ham_e}\qd$ is energy conserving when
\begin{align}\label{eqn:EnergizationTransformAlpha}
    \alpha_{\Ham_e} = -(\qd^\tr\M\qd)^{-1}\qd^\tr\big[\M\pi + \bm{\xi}\big],
\end{align}
and differs from the original system by only an acceleration along the direction of motion. 
The new system is a geometric fabric, known as an {\em energized fabric}, expressed as:
\begin{align} \label{eqn:ZeroWorkEnergizationForm}
  \M\qdd + \bm{\xi} - \mP_e\big[\M\pi + \bm{\xi}\big] = \zero,
\end{align}
where $\mP_e = \M^{\,\frac{1}{2}}\Big[\I - \hat{\vv} \hat{\vv}^\tr\Big]\M^{-\frac{1}{2}}$ is a metric-weighted projection matrix
with $\vv = \M^{\,\frac{1}{2}}\xd$ and $\hat{\vv} = \vv/\|\vv\|$. 
\end{theorem}
\begin{proof}
See \apx~\ref{apx:TheoremOnEnergizedFabrics}.
\end{proof}
We denote the energization operation using an operator $\mathrm{energize}_{\Lag_e}[\cdot]$ that outputs a geometric fabric given an HD2 geometry $\pi$ so that 
\begin{align}
    \mathrm{energize}_{\Lag_e}[\pi] = \Big(\Lag_e,\: -\mP_e\big[\M\pi+\bm{\xi}\big]\Big)    
\end{align}
is the energized geometry. Note that the energization operator can be applied to a non-geometric policy as well. In that case, we write the output in the same form using the same energization coefficient calculation in Equation~\ref{eqn:EnergizationTransformAlpha}, but with the knowledge that the resulting transformed system will conserve energy but not be a geometric fabric.

An informative example is when $\Lag_e = \frac{1}{2}\|\qd\|^2$ measures the squared velocity norm. Here $\M = \I$ and $\bm{\xi} = \zero$, so under Equation~\ref{eqn:EnergizationTransformAlpha} the system reduces to 
\begin{align}
    \qdd = \left[\I - \widehat{\qd}\,\widehat{\qd}^\tr\right]\pi(\q, \qd),
\end{align}
where $\widehat{\qd} = \qd / \|\qd\|$ is the normalized velocity.
In other words, the energized system projects accelerations of the geometric policy $\pi$ orthogonal to the direction of motion so they curve the system without changing its speed.

\subsection{Constraints in generalized coordinates}
\label{sec:GeneralizedCoordinates}

We showed above that we can construct fabrics by energizing geometries (see Theorem~\ref{thm:SystemEnergization}) and constraining them (see Theorem~\ref{lma:ConstrainedGeometricFabrics}).
There we calculated explicit constraint forces in the ambient space for analysis, but in many cases, it's natural to express constrained systems in generalized coordinates such as the joint angles of the robot. This section presents a fundamental result showing how to calculate an energized policy in generalized coordinates.

First we formally derive how to constrain systems using generalized coordinates. Let $\M(\x,\xd)\xdd + \f(\x,\xd) = \zero$ be any system of this form defined on a space $\mathcal{X}$ ($\M$ need not derive from a Finsler energy and $\f$ captures all system forces). Suppose $\phi:\mathcal{Q}\rightarrow\mathcal{X}$ defines a surface in $\mathcal{X}$ with generalized coordinates $\mathcal{Q}$ and Jacobian $\J=\partial_\q\phi$ (note that the rows of $\J$ here are orthogonal to those of the implicit surface Jacobian due to the differing representation). We have $\x = \phi(\q)$, $\xd = \J\qd$, and $\xdd = \J\qdd + \Jd\qd$; plugging that last expression for $\xdd$ into the system equation and multiplying through by $\J$ gives an expression for the constrained system in generalized coordinates:
\begin{align} \label{eqn:PullbackDerived}
    \Big(\J^\tr\M\J\Big) \qdd + \J^\tr\big(\f + \M\Jd\qd\big) = \zero.
\end{align}
This equation has the same form as the original system in $\mathcal{X}$, but with $\M\rightarrow \J^\tr\M\J$ and $\f\rightarrow \J^\tr\big(\f + \M\Jd\qd\big)$. This transformation is known as the {\em pullback} of the system into the generalized coordinates, and can be expressed as an operation on system parameters
\begin{align} \label{eqn:SpecPullback}
    \mathrm{pull}_\phi\big[(\M,\f)\big] = \Big(\J^\tr\M\J,\,\J^\tr\big(\f + \M\Jd\qd\big)\Big).
\end{align}
For Lagrangian systems, this pullback operation is consistent with the generalized coordinate calculation of the Euler-Lagrange equation. Specifically, if $\Lag(\x,\xd)$ is any Lagrangian in $\mathcal{X}$, we can define
\begin{align}
    \mathrm{pull}_\phi\big[\Lag\big] = \Lag\Big(\phi(\q), \frac{d}{dt}\phi(\q)\Big).
\end{align}
Then if $\mathrm{EL}[\Lag] = \big(\M,\bm{\xi}\big)$ denotes the application of the Euler-Lagrange equation to construct equations of motion $\M\xdd + \bm{\xi} = \zero$, we have
\begin{align} \label{eqn:PullbackCommutesWithEL}
    \mathrm{pull}_\phi\big[\mathrm{EL}[\Lag]\big]
    = \mathrm{EL}\big[\mathrm{pull}_\phi[\Lag]\big]
    = \mathrm{pull}_\phi\big[(\M,\bm{\xi})\big].
\end{align}

The following theorem reveals that a constrained energized system in generalized coordinates can be constructed by first constraining the metric-weighted geometry then energizing directly in the generalized coordinates. 

\begin{theorem}\label{thm:FabricsInGeneralizedCoords}
Let $\Lag_e^\mathcal{X}(\x,\xd)$ be a Finsler energy on $\mathcal{X}$ with equations of motion $\M\xdd + \bm{\xi} = \zero$ and let $\pi(\x,\xd)$ be an HD2 geometry on $\mathcal{X}$. If $\phi:\mathcal{Q}\rightarrow\mathcal{X}$ is a full rank differentiable map defining a constraint in $\mathcal{X}$ with generalized coordinates $\mathcal{Q}$, then 
\begin{align}\nonumber
    \mathrm{pull}_\phi\big[\mathrm{energize}_{\Lag_e^\mathcal{X}}[\pi]\big]
    = \mathrm{energize}_{\Lag_e^\mathcal{Q}}
        \Big[\mathrm{pull}_\phi\big[\big(\M, -\M\pi\big)\big]\Big],
\end{align}
where $\Lag_e^\mathcal{Q}=\mathrm{pull}_\phi[\Lag_e^\mathcal{X}]$ and $\big(\M, -\M\pi\big)$ denotes the parameters of the metric weighted geometry $\M\big(\xdd - \pi(\x,\xd)\big) = \zero$.
\end{theorem}
\begin{proof}
See \apx~\ref{apx:TheoremFabricsInGeneralizedCoords}.
\end{proof}

\section{Designing with fabrics}
\label{sec:DesigningFabrics}

Theorem~\ref{thm:FabricsInGeneralizedCoords} is fundamental to fabric design because it enables us to design using transform trees. A transform tree is a tree of differentiable maps, where each node is a task space and each edge is a differentiable map mapping a parent space (domain) to a child space (co-domain). The tree-structure expresses sparsity patterns important for computational efficiency, but conceptually we can think of the tree as a single differentiable map linking the root space to the stacked joint output space of all nodes. From that view it's clear that we can exploit Theorem~\ref{thm:FabricsInGeneralizedCoords} to say that the resulting constrained system constraining the joint ambient space fabric is equivalent to the system defined by pulling the ambient space fabric to the root. 
In light that, it's often useful to design geometric fabrics in parts, viewing each part as an HD2 geometric policy $\pi$ with an associated priority metric defined by a Finsler energy $\Lag_e$, each living on a node space of a transform tree. Then computationally we can pull the pieces back to the root exploiting sparsity in the same way as in earlier work \cite{cheng2018rmpflow}.
When designing we can reason about what the geometry of each part should be and how we should prioritize it, 
understanding that they will combine a metric-weighted average of geometries.

Let $\mathcal{G} = \{\Lag, \f\}$ compactly denote a system of the form $\M\xdd + \f = \zero$ paired with a Finsler energy from which $\M$ derives. We call $\mathcal{G}$ a {\em fabric component}. Note that this system by itself is not a geometric fabric, which must be a bent Finsler system. Here $\f$ can be arbitrary. For instance, it often denotes $\f = -\M\pi$ derived from geometric policy $\pi$.

The scaled sum of two fabric components $\mathcal{G}_1 = \{\Lag_1,\f_1\}$ and $\mathcal{G}_2 = \{\Lag_2,\f_2\}$ is given by
\begin{align}
    \alpha \mathcal{G}_1 + \beta \mathcal{G}_2
    = \big\{\alpha \Lag_1 + \beta \Lag_2, \alpha \f_1 + \beta \f_2\big\},
\end{align}
where $\alpha,\beta\in\R_+$. Likewise, the pullback of a fabric component across a map $\x = \phi(\q)$ with Jacobian $\J$ is
\begin{align}
    \mathrm{pull}_\phi\big[\{\Lag, \f\}\big]
    = \Big\{\mathrm{pull}_\phi[\Lag], \J^T\big(\f + \M\Jd\qd\big)\Big\}.
\end{align}
\fullversiononly
See \apx~\ref{apx:DerivationsOfFabricComponentAlgebra} for formal derivations of these formulas. 
\fi
These expressions are easily derived using the basic definitions of Finsler energy pullback and fabric pullback from Section~\ref{sec:GeneralizedCoordinates}. It is easy to show that summation is linear (which implies associativity and commutativity), and pullback across the composition of functions is equal to the pullback across each $\mathrm{pull}_{\phi_2\circ\phi_1}[\mathcal{G}] = \mathrm{pull}_{\phi_1}[\mathrm{pull}_{\phi_2}[\mathcal{G}]]$. 

Given a transform tree, we populate the nodes with these fabric components and exploit this algebra to 
recursively pull each component from child to parent and sum the results at the parent (with optional scalar weights). What results is a final fabric component at the root $\wt{\mathcal{G}} = \{\wt{\Lag}, \wt{\f}\}$ and associated Finsler energy $\wt{\Lag}$. $\wt{\mathcal{G}}$ represents system $\wt{\M}\qdd + \wt{\f} = \zero$ in generalized coordinates defining an HD2 geometric policy $\wt{\pi} = -\wt{\M}^{-1}\wt{\f}$ since it is formed as a metric-weighted average of pullback geometries from each component (and the metrics are HD0). The final fabric is simply the energization of this geometry with the associated Finsler energy:
\begin{align}
    \mathcal{F} = \mathrm{energize}_{\wt{\Lag}}\big[\wt{\pi}\big].
\end{align}
In practice, we also leverage the geometric path consistency of the final geometric policy to regulate speed without affecting the path traced during execution (see \cite{ratliff2021geometry}).

\fullversiononly
\subsection{Speed control through damping regulation}
\label{sec:SpeedControl}

Because we have the flexibility now to encode the nominal behavior into the fabric itself we can leverage accelerations along the direction of motion (known to leave the geometric paths unchanged) to design dampers that regulate speed without affecting the behavior.
We use dampers of the form
\begin{align}
    \f_\mathrm{damp} = -\beta(\q,\qd)\M(\q,\qd)\qd,
\end{align}
with $\beta\in\R_+$ so that the resulting acceleration $\M^{-1}\f_\mathrm{damp} = -\beta\qd$ acts solely along the direction of motion. 
The potential then has a dual role to (1) speed the system (the acceleration component parallel to the direction of motion) and (2) force the system off a nominal path toward the goal (the acceleration component orthogonal to the direction of motion). 
See Appendix~\ref{apx:FormulasSpeedControl} for a discussion on how to calculate $\beta$ to effectively control a target measure of speed.

\subsection{Tools for fabric design}
\label{sec:ToolsForFabricDesign}

Finsler energies and HD2 geometries share a common requirement of being Homogeneous of Degree 2 (HD2). \apx~\ref{apx:DesigningHD2Terms} presents a number of methods for designing such functions which we use throughout our experiments. At a high-level these tend to revolve around designing HD0 terms (which can rely on velocity directionality but not the norm (speed)), and scaling them by some measure of Finsler energy (such as the squared norm of the velocity) to bring them from HD0 to HD2. We use these techniques to design the fabrics used in our experiments as described in Sections \ref{sec:ControlledParticleExperiments} and  \ref{subsec:FabricsDesign}.
\fi




\section{Particle Experiments} \label{sec:ControlledParticleExperiments}
Geometric fabrics have two parts to their generalization: (1) moving from Riemannian geometries to Finsler geometries; (2) bending those Finsler geometries. These experiments explore how both of those pieces affect the system by comparing four types of systems: (i) strictly Riemannian systems (classical mechanical systems); (ii) standard Finsler systems (enabling velocity-dependent metrics, but without bending terms); (iii,iv) bent variants of each of the first two. In the discussion below, we often refer to both Riemannian and Finsler systems as {\em unbent fabrics} and the bent variants as geometric fabrics; there are four combinations in total: unbent Riemannian, unbent Finsler, bent Riemannian, bent Finsler (this last one being a full geometric fabric).

We start with a collection of particles at rest to the right of a circular obstacle and we observe the behavior of different variants controlled to two different speeds attracting the particles to a point to the left of the obstacle (small square) under the influence of obstacle avoidance.

\subsection{Point Attraction}
\label{subsec:point_attraction}
This term uses task map $\x = \phi(\q) = \q - \q_d$ where $\q$, $\q_d \in \mathbb{R}^2$ are the current and desired particle position in Euclidean space. The metric is designed as $\G_\psi(\x) = (\widebar{m} - \underline{m}) e^{-(\alpha_m \|\x\|) ^ 2} \I + \underline{m} \I$, where $\widebar{m}$, $\underline{m} \in \mathbb{R}^+$ are the upper and lower isotropic masses, respectively, and $\alpha_m \in \mathbb{R}^+$ controls the width of the radial basis function. 
We design the potential gradient directly as $\partial_\q \psi(\x) = \M_\psi(\x) \partial_\q \psi_1(\x)$, with $\psi_1(\x) = k \left( \|\x\| + \frac{1}{\alpha_\psi}\log(1 + e^{-2\alpha_\psi \|\x\|} \right)$, where $k \in \mathbb{R}^+$ controls the overall gradient strength, $\alpha_\psi \in \mathbb{R}^+$ controls the transition rate of $\partial_\q \psi_1(\x)$ from a positive constant to 0. 
For the unbent fabric, the system derives from $\Lag = \xd^T \G_\psi(\x) \xd - \psi(\x)$. For the geometric fabric, we use fabric component defined by geometry $\xdd = -\partial_\x \psi_1(\x)$ and energy $\Lag_e = \xd^T \G_\psi(\x) \xd$.

\subsection{Circular Object Repulsion}
\label{subsec:circular_object_repulsion}
This term uses task map $x = \phi(\q) = \frac{\|\q - \q_o\|}{r} - 1$, where $\q_o$ is the origin of the circle and $r$ is its radius. Two different metrics are designed to prioritize object repulsion, $G_b (x) = \frac{k_b}{x^2}$ and $G_b (x, \dot{x}) = s(\dot{x}) \frac{k_b}{x^2}$. Moreover, $k_b \in \mathbb{R}^+$ is a barrier gain and $s(\dot{x}) = 1$ if $\dot{x} < 0$ and $s(\dot{x}) = 0$, otherwise. 
We design the potential gradient as $\partial_\q \psi_b(x) = M_b (x) \partial_\q \psi_{1,b}(x)$ with $\psi_{1,b}(x) = \frac{\alpha_b}{2 x^8}$, and $\alpha_b \in \mathbb{R}^+$
For the unbent fabric, the system derives from either $\Lag_1 = G_b(x) \dot{x}^2 - \psi_b(x)$ (Riemannian) or  $\Lag_2 = G_b(x, \dot{x}) \dot{x}^2 - 
\psi_b(x)$ (Finsler). For the geometric fabric, we use fabric component defined by $\ddot{x} = - s(\dot{x}) \dot{x}^2 \partial_x \psi_{1,b}(x)$ and energy either $\Lag_1$ or $\Lag_2$.

\subsection{Discussion}
\label{subsec:particle_discussion}

Fig. \ref{fig:gf_vs_lds_particles} visualizes the results of all for variants each under two different speeds $v_d=2,4$ to analyze degree of consistency under speed changes. For each variant, all particles avoid the object and reach the target position except the one shot directly at the object's local minimum. However, we observe several differences in behavior. 

First, the unbent fabrics (left column) have more pronounced changes in paths across the different speed levels. This is amplified for Finsler systems (bottom two rows) using $G_b(x, \dot{x})$ since the velocity gate does not modulate the obstacle avoidance policy. Instead, the mass of the obstacle avoidance policy vanishes, while components of its force remain. This effect is amplified when traveling at a higher velocity, producing the ``launching'' artifacts. In contrast, geometric fabrics (right column) produce more consistent paths across speed levels without any launching artifacts. 

Second, Finsler metrics (bottom two rows) facilitate straight-line motion to the desired location once past the obstacle by allowing the system to forget about it as a function of directionality. Riemannian metrics are unable to represent that directional dependence (top two rows) and can be seen clinging to the obstacle even once past it.

Finally, the repulsive forces from the unbent fabrics (right column) consistently push the center particle farther out from the object than with geometric fabrics (right column). Importantly, geometric fabrics allow for heightened obstacle avoidance behavior without shifting the system minima.

\begin{figure}[!t]
  \centering
  \fullversiononly
  \includegraphics[width=.8\linewidth]{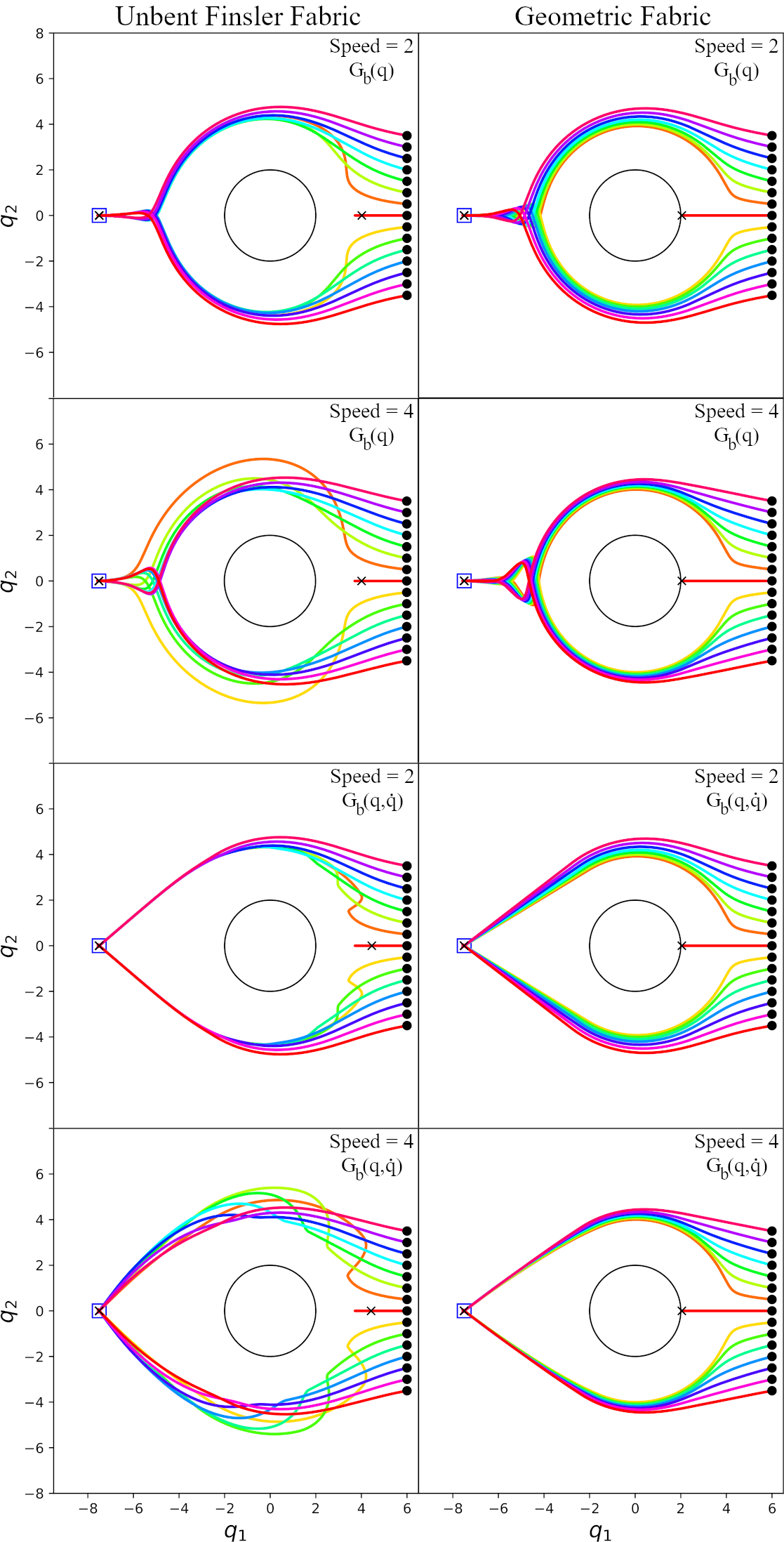}
  \else 
  \includegraphics[width=.6\linewidth]{figures/lds_vs_gf.png}
  \fi
  \vspace{-4mm}
  \caption{Particle behavior for an unbent fabric and a geometric fabrics with different metric designs. See Section \ref{subsec:particle_discussion} for discussion. }
  \vspace{-6mm}
  \label{fig:gf_vs_lds_particles}
\end{figure}

\section{Franka Panda Experiments}
\label{sec:SystemExperiments}

\compressedversiononly
\begin{figure*}[!t]
  \centering
  \includegraphics[height=.15\linewidth]{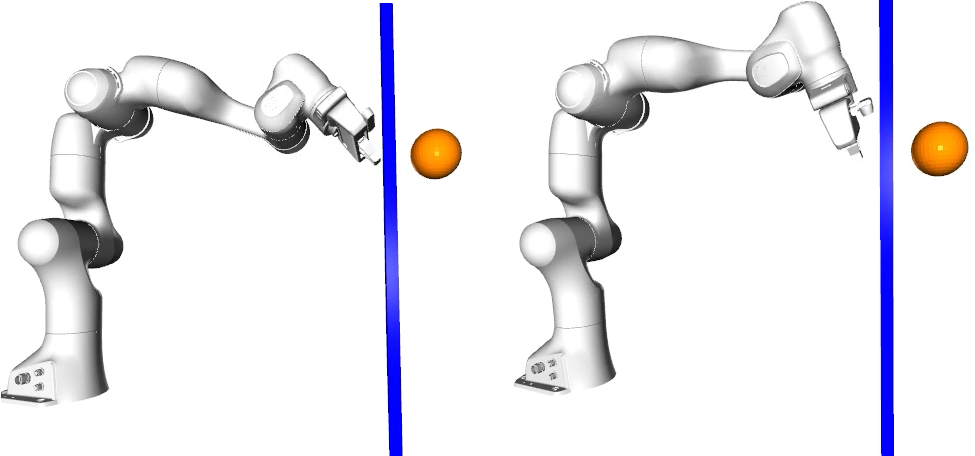}
  \hfill
  \includegraphics[height=.15\linewidth]{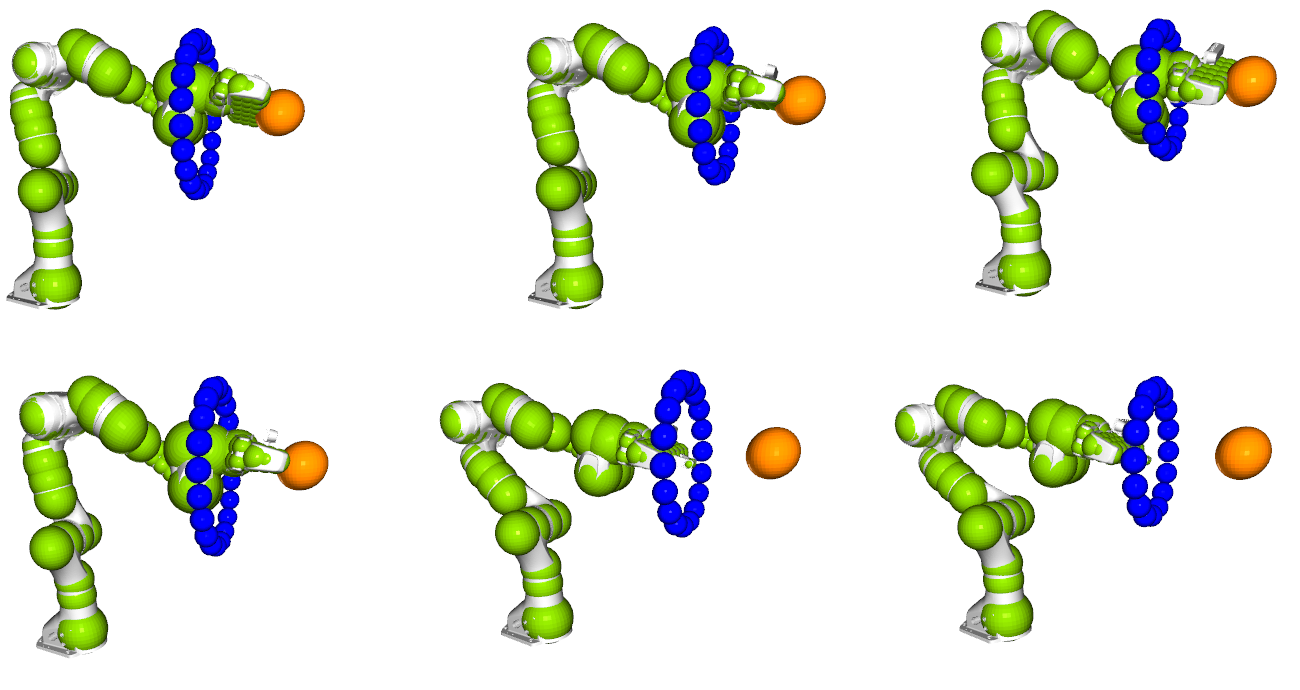}
  \hfill
  \includegraphics[height=.15\linewidth]{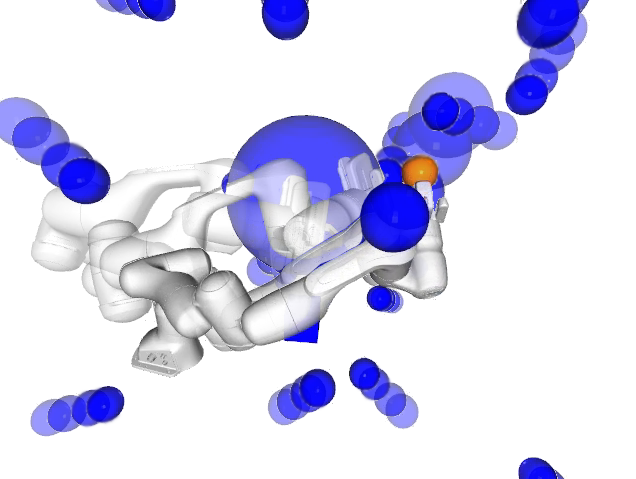}
  \hfill
  \vspace{-4mm}
  \caption{Left: Reaching a target blocked by a wall (see Section~\ref{subsec:WallBarrierApproach}). Middle: Ring reaching experiment comparing fabrics (top) and RMPs (bottom) (see Section~\ref{subsec:RingNav}). Right: dynamic obstacle avoidance experimental setup (see Section~\ref{subsec:DynamicReaching}).}
  \vspace{-4mm}
  \label{fig:experiments}
\end{figure*}
\fi

Our second set of experiments is a study of geometric fabrics on a full 7 degree-of-freedom robot arm. The geometric fabric is designed once, and tested across a variety of problems. We compare its performance to our current, best performing Riemmanian Motion Policy (RMP) that has powered our manipulation processes for a variety of settings. Both policies are evaluated at 100 Hz and integrated with first-order Euler routines. Note that since this is purely a motion generation problem, the sim-to-real gap is minimal since smooth trajectories can easily be executed on physical robots using standard trajectory following controllers. We, therefore, focus only a simulated analysis here.

\subsection{Geometric Fabrics Design}
\label{subsec:FabricsDesign}
The geometric fabric policy ultimately consists of only three differently designed components. Each component consists of a Finsler energy $\Lag_e(\x, \xd)$ and policy $\pi(\x, \xd)$.

\subsubsection{Joint Space Attraction}
The Finsler energy is designed as $\Lag_e = \frac{m}{2} \|\xd\|^2$, where $m \in \mathbb{R}^+$ is the amount of isotropic mass. The policy is designed as $\pi = -\|\xd\|^2 \partial_\x \psi(\x)$ with $\psi(\x) = k \log \left(\cosh (\alpha (\x_d - \x))\right) $ where $k$, $\alpha \in \mathbb{R}^+$ are gains that control scaling and sharpness, respectively, and $\x$ and $\x_d$ are the current and desired joint space positions. This component acts to posture the arm largely within the null space of higher-priority tasks such as target reaching.

\subsubsection{Distance Space Repulsion}
The Finsler energy is designed as $\Lag_e = \frac{k}{2 x} s(\dot{x}) \dot{x}^2$, where $k \in \mathbb{R}^+$ scales the barrier function and $s(\dot{x}) = 1$ if $\dot{x} < 0$ and $s(\dot{x}) = 0$, otherwise. 
We define a potential 
$\psi(x) = \frac{k_b}{x} + \frac{k_r}{\alpha} \log \left(1 + e^{-\alpha (x - x_o)} \right)$, where $k_b$, $k_r \in \mathbb{R}^+$ control the scaling of the barrier and soft rectified linear unit (ReLU) functions, respectively and $x \in \mathbb{R}^+$ is a signed distance space (e.g., distance to object). Moreover, $\alpha \in \mathbb{R}^+$ controls the sharpness in the ReLU and $x_o \in \mathbb{R}^+$ shifts the activation of the ReLU. We deploy this potential alongside a geometric variant $\pi = -\dot{x}^2 \partial_x \psi(x)$, for joint-limit avoidance, self-collision avoidance, and object-collision avoidance. The geometric policies are weighed significantly higher than potentials, which play the role of soft-penalty on the obstacle constraint here. The potential policy is included to shift the system minima away from numerical singularities, e.g., as $x \to 0$, then $\dot{x} \to 0$ and $\partial_x \psi(x) \to \infty$, which generates numerical issues for calculating the geometric policy.

\subsubsection{End-effector Space Attraction}
\label{subsubsec:attraction}
Using $\overline{m}$, $\underline{m} \in \mathbb{R}^+$ to denote maximum and minimum isotropic masses, respectively,
the Finsler energy is designed as $\Lag_e = \|\xd\|^2\left( \frac{1}{2} (\overline{m} - \underline{m}) (\tanh(-\alpha \|\x_d - \x\|) + 1) + \underline{m} \right)$, where $\alpha \in \mathbb{R}^+$ controls the sharpness of the hyperbolic tangent function, and $\x,\x_d \in \mathbb{R}^3$ are the position and target position of the end-effector, respectively. The 
potential is defined as 
$\psi(\x) = k \log \left(\cosh (\alpha (\x_d - \x))\right) $, where $k,\alpha \in \mathbb{R}^+$ are gains that control scaling and sharpness, respectively. We deploy both the potential and a geometric variant $\pi = -\|\xd\|^2 \partial_\x \psi(\x)$. Each are individually tuned, with the geometric variant being generally more active and the potential there to ensure convergence in alignment with our theory. Applying this design to two axes at the end-effector can also realize orientation control as in \cite{ratliff2018rmps}.

\subsection{Wall Barrier Approach} \label{subsec:WallBarrierApproach}

\fullversiononly
\begin{figure}[!t]
  \centering
  \includegraphics[width=.9\linewidth]{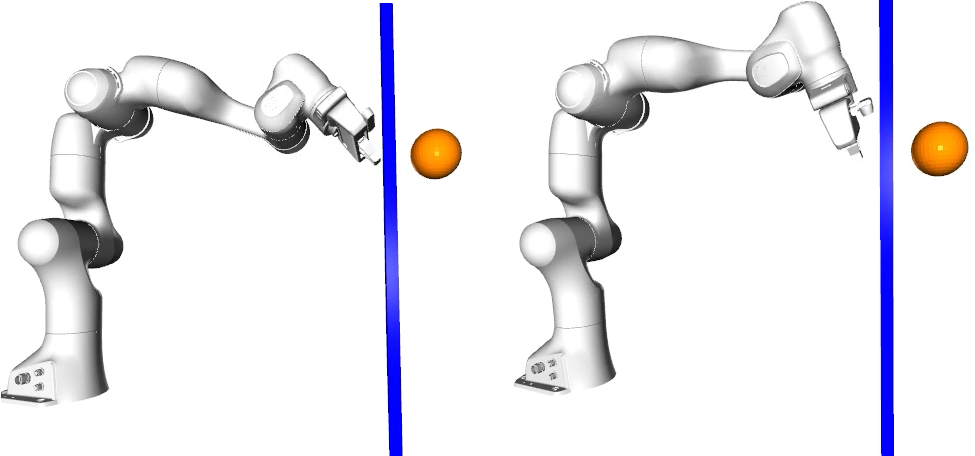}
  \vspace{-4mm}
  \caption{Robot reaching for a target (orange sphere) while blocked by an obstacle (blue) under a geometric fabric policy (left) and an RMP (right).}
  \label{fig:wall_approach}
\end{figure}
\fi

To assess how well the policies can reach 3D end-effector position targets in the presence of boundary constraints, we incrementally advance the end-effector target every five seconds towards a wall boundary and monitor convergence 
\compressedversiononly
(see Fig. \ref{fig:experiments} (left)).
\else 
(see Fig. \ref{fig:wall_approach}).
\fi
Ideally, a policy should accurately converge to the target while respecting the boundary constraint, but in practice boundary potentials can prevent that. In this experiment, we see geometric fabrics converge more accurately in general to these targets than the RMP policy. Both policies exhibit sub-millimeter convergence error when the target is 10 cm or more from the wall, but convergence incrementally degrades as the target approaches the wall due to conflicting potentials. RMPs already start at degrading at $\sim$9 cm while geometric fabrics begin degrading only at $\sim$6 cm since much of the boundary policy is encoded in a geometric policy. Moreover, geometric fabrics on average approach the wall about 50\% closer than the RMP. 
\compressedversiononly
Figure~\ref{fig:experiments} (left) 
\else
Fig.~\ref{fig:wall_approach}
\fi
depicts this difference in convergence near the wall.

\subsection{Ring Constricted Navigation} \label{subsec:RingNav}

\fullversiononly
\begin{figure}[!t]
  \centering
  \includegraphics[width=.9\linewidth]{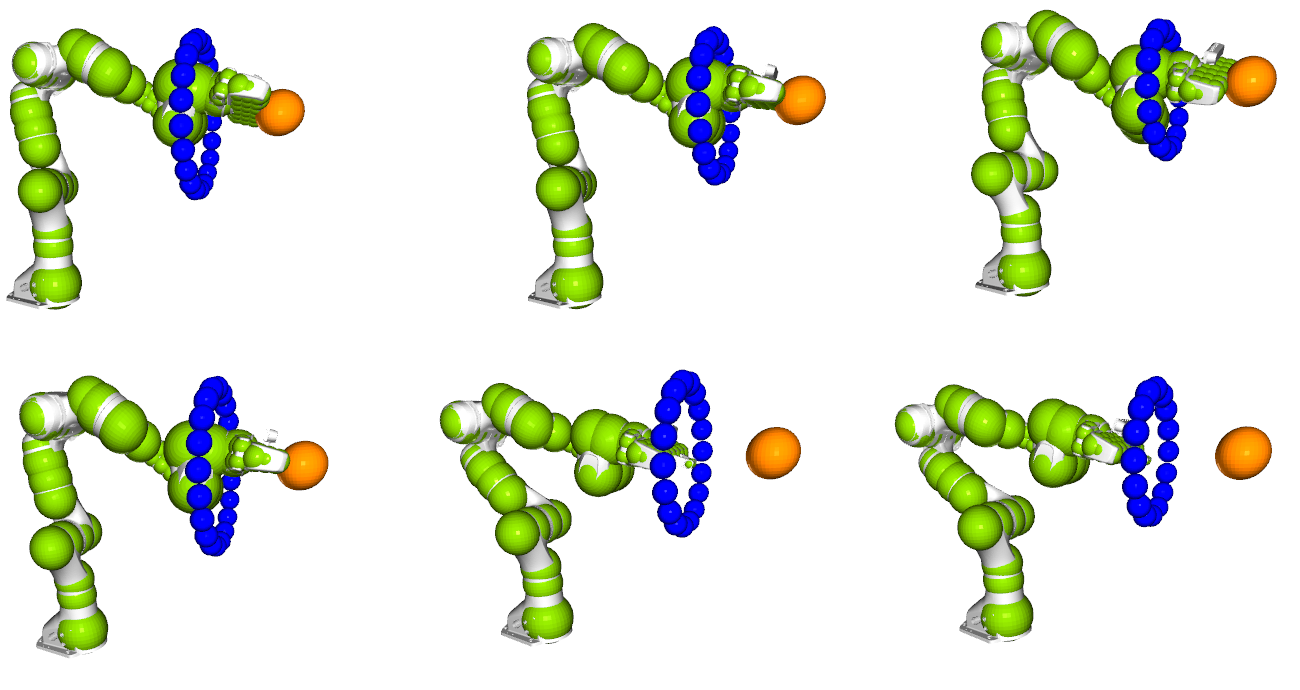}
  \vspace{-4mm}
  \caption{Robot executing reach to target (orange spheres) through a series of obstacle rings (blue) of increasingly smaller diameter with a geometric fabric (top row) and RMP (bottom row). Green spheres approximate robot geometry for collision avoidance. }
  \label{fig:ring_navigation}
\end{figure}
\fi

To assess convergence to end-effector targets in the presence of increasingly constricted passages, we create a sequence of rings with incrementally smaller radii that the end-effector must pass through to reach its target. The ring radius starts at 17 cm and decreases to 13 cm. The rings are constructed from small spheres (each with radius 2.5 cm) as shown in 
\compressedversiononly
Figure \ref{fig:experiments} (middle). 
\else
Fig.~\ref{fig:ring_navigation}.
\fi
Distance spaces are pair-wise between robot spheres and obstacle spheres. As shown, the navigation from a geometric fabric allows the robot to successfully pass through all rings and reach its target. In contrast, the RMP cannot pass through the last two rings of the smallest diameter. Again, this difference arises with primarily encoding avoidance into the fabric's geometry alongside a diminished repulsive potential function.

\subsection{Dynamic Obstacle Reaching} \label{subsec:DynamicReaching}

\fullversiononly
\begin{figure}[!t]
  \centering
  \includegraphics[width=.8\linewidth]{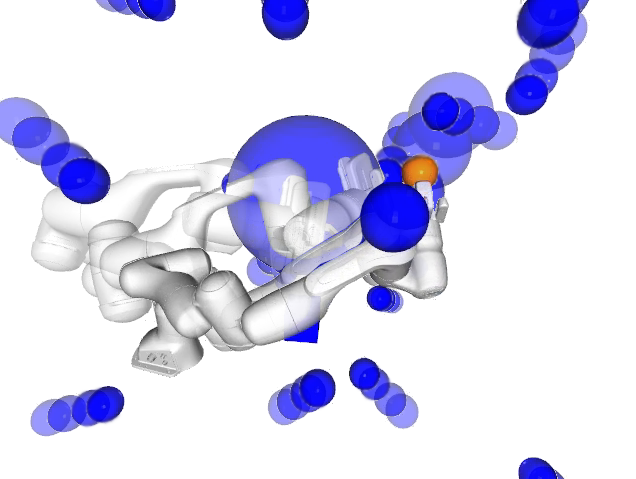}
  \vspace{-4mm}
  \caption{Robot reaching end-effector targets while avoiding dynamic objects.}
  \label{fig:dynamic_navigation}
\end{figure}
\fi

By nature, both the geometric fabric policy and RMP are reactive. To assess the quality of reactivity, we create a reaching task with both static and dynamic obstacles. The goal is for the policies to reach 19 randomly generated end-effector targets (with one target issued every 5 seconds) while avoiding collision with both static and dynamic obstacles 
\compressedversiononly
(see Fig. \ref{fig:experiments} (right)). 
\else
(see Fig. \ref{fig:dynamic_navigation}). 
\fi
The policies are only aware of the object positions while objects achieve speeds of up to $0.39$ $m/s$. With the geometric fabric policy, the robot reached 16 targets whereas the robot with the RMP reached 11 targets. Moreover, the geometric fabric policy yielded a $0.4 \%$ collision rate in contrast to the RMP's $8.18 \%$,  where collision rate is the percentage of time steps in collision. Since the geometric fabric obstacle avoidance component is predominantly geometric in nature, it can simultaneously be highly influential while not preventing optimization of the end-effector potential, resulting in higher target acquisition rates. The forcing component of the geometric fabric policy aggressively engages only when the robot is close to collision. Altogether, this allows strong obstacle avoidance behavior both close to and far from objects in a way that minimally impedes target acquisition. In contrast, the RMP had to strike a balance between forcing repulsion without impairing the ability to acquire end-effector targets too much. This compromise resulted in diminished performance.

\section{Conclusion}
Geometric fabrics naturally generalize classical mechanical systems to form a new physics of behavior. These systems provide important design expressivity while maintaining the theoretical properties that make classical systems powerful in control. Many existing techniques in the literature can naturally leverage fabrics (e.g., reinforcement learning as in \cite{LiRmp2RobotLearning2021}, imitation learning, planning), and it is a point of future work to explore many of these areas. Our theoretical and experimental analyses already demonstrate their utility but we believe the bulk of their application has yet to come.
\compressedversiononly
For extended discussions, proofs, and additional exhibitions of geometry fabrics (including full pose control and equality constraint handling), see the extended version \cite{Wyk2021GeometriFabricsPhysicsArXivAbbrev}, video, and website (\url{https://sites.google.com/nvidia.com/geometric-fabrics-behavior}).
\else 
\fi



{
\small
\bibliographystyle{IEEEtran}
\bibliography{refs}
}


\begin{appendices}
\title{Appendices}
\date{}
\maketitle

\section{Preliminaries}
\subsection{Derivations on manifolds}
\label{apx:Manifolds}

Nonlinear geometry is most commonly developed in the mathematics literature in terms of smooth manifolds, and often in an abstract coordinate-free notation which can be challenging for practitioners. Even tensor notation can be a barrier to entry for readers unfamiliar with the algebra. To make our presentation more accessible, we will stick to coordinate descriptions and simpler vector notations from advanced calculus (see Appendix~\ref{apx:CalculusNotation}) similar to many books on classical mechanics aimed at an engineering audience. 

Manifolds will be formally defined in the traditional way (see \cite{LeeSmoothManifolds2012} for a complete introduction), but for our purposes, we will consider them simply as $d$-dimensional spaces $\mathcal{Q}$ with elements identified with $\q\in\R^d$ in $d$ generalized coordinates, and when we say $\q\in\mathcal{Q}$ we mean $\q$ is the coordinate representation of the manifold point described in our chosen coordinates. Often we implicitly assume a system evolves over time $t$ in a trajectory $\q(t)$ with velocity $\qd = \frac{d\q}{dt}$, and say that coordinate velocity vector $\qd$ is an element of the tangent space $\mathcal{T}_\x\mathcal{Q}$, which can be thought of as the space of all velocity vectors that can be generated in that way. 
When discussing general tangent space vectors (elements of the tangent space without reference to a specific trajectory), we often use the notation $\vv\in\mathcal{T}_\q\mathcal{Q}$.

\subsection{A note on covariance and transform trees}
\label{apx:NoteOnCovariance}

Geometric consistency across changes of coordinates on a manifold is an important concept since many manifolds cannot be described globally by a single set of coordinates. Algorithms and systems should not change their fundamental behavior on the manifold based on the coordinate system they're described in. When an algorithm or system has such a coordinate independence, it's called {\em covariant}.

It is well known that the Euler-Lagrange equation is covariant to changes of coordinate \cite{LeonardSusskindClassicalMechanics}, so anything derived here from the Euler-Lagrange equation is automatically covariant. For everything else, such as HD2 geometries and bending terms, we are not rigorous with developing specific intrinsically covariant subclasses. Such covariant subclasses exist and can be derived, but the mathematics behind that can be complicated and unnecessarily restrictive. In practice, we instead use transform trees to derive covariant systems.

A transform tree is simply a tree whose nodes are different spaces and whose edges are the maps linking those spaces. For instance, the simplest tree is a single branch $\phi:\mathcal{Q}\rightarrow\mathcal{X}$ mapping generalized coordinates $\mathcal{Q}\subset\R^d$ to a single space $\mathcal{X}\subset\R^n$. In fact, all trees can be expressed in this form since paths from the root to any given node are unique and the maps encountered along that path can be combined into a single composed map.
The single branch tree is constructed by stacking all nodes into a single space and defining a map by stacking these unique path-composed maps into a single map. In practice, more expressive tree structures are important for computational performance, but conceptually it suffices to think of simply a single link $\phi$ as described here. (Note that $n\geq d$ when the tree is well-formed so that policies on the nodes can fully constrain the space.)

Equation~\ref{eqn:PullbackCommutesWithEL} shows that the pullback operation defined on a transform tree is compatible with the pullback implicit in the covariant Euler-Lagrange equation. This is because pullback is compatible with the definition of covariance. Calculationally, we prove that a system is covariant if, under a change of coordinates, the resulting system is simply scaled by the Jacobian (transpose) of the change of coordinates map. For instance, if $\phi$ denotes a change of coordinates (i.e. when $n = d$ and $\phi$ is smooth and invertible making it what is known as a {\em diffeomorphism}), a term $\h(\x,\xd, \xdd)$ is covariance if its expression in $\mathcal{Q}$ amounts to $\wt{\h}(\q,\qd,\qdd) = \J^\tr \h(\x,\xd,\xdd)$ where $\J = \partial\phi$ is the Jacobian of $\phi$ and $\Jd$ is its time derivative (and $\x = \phi(\q), \xd = \J\qd, \xdd = \J\qdd + \Jd\qd$). The utility of this definition of covariance is that if $\wt{\h}(\q,\qd, \qdd) = \zero$ then in $\mathcal{Q}$ we get the same 
\begin{align}
    &\wt{\h}(\q,\qd,\qdd) =\zero \\\nonumber
    &\ \ \Rightarrow  \J^\tr\h(\x,\xd,\xdd) = \zero \\\nonumber
    &\ \ \Rightarrow \h(\x,\xd,\xdd) = \zero
\end{align}
since $\J$ is invertible when $\phi$ is a change of coordinates. In other words, covariant equations hold in all coordinate systems.

Now, looking at Equation~\ref{eqn:SpecPullback} and its derivation leading to Equation~\ref{eqn:PullbackDerived}, we see that when $\phi$ is a change of coordinates, the system transforms covariantly by definition. Specifically, when the term is $\h(\x,\xd,\xdd) = \M(\x,\xd)\xdd + \f(\x,\xd)$, the pullback operation computes $(\J^\tr\M\J\big)\qdd + \big(\f + \M\Jd\qd\big) = \J^\tr\big(\M\xdd + \f\big) = \J\,\h(\x,\xd,\xdd)$ (we're just reversing the argument used to derive it).
Therefore, if we have any arbitrary transform tree $\phi:\mathcal{Q}\rightarrow\mathcal{X}$ (viewed as a single branch) and we want to change the coordinates using $\zeta:\wt{\mathcal{Q}}\rightarrow\mathcal{Q}$, we can simply build a transform tree linking these transforms
\begin{align}
    \wt{\mathcal{Q}}\stackrel{\zeta}{\longrightarrow} \mathcal{Q}
    \stackrel{\phi}{\longrightarrow}\mathcal{X}
\end{align}
and use standard pullback operations to get covariance automatically. For instance, if $\mathcal{Q}$ is the generalized coordinates of the robot, we can define a transform tree with $\mathcal{Q}$ at the root and a collection of task spaces extending from it such as the end-effector space, body points, distances to obstacles, etc. Then, if we change to a different coordinate representation of the generalized coordinates we simply express that as part of the transform tree and the fundamental behavior of our system does not change because of the natural covariance of the transform tree itself. An important example is when the generalized coordinates includes a free flying SE3 base such as in a drone with an SO3 rotation component. SO3 cannot be globally expressed in a single coordinate system, so we inevitably need to change coordinates at times. (We can also change representations, between Euler angles, quaternions, axis-angle and matrix representations, as long as the transform is represented as part of the tree.) Covariance ensures that the system behavior remains consistent under these representational transformations.

Another way to understand the transform tree is that it represents an optimization problem. Pullback can equivalently be viewed as the solution to a least squares optimization problem
\begin{align}
    &\min_\qdd \frac{1}{2}\|\xdd^d - \xdd\|_\M^2 \\\nonumber
    &\mathrm{s.t.}\ \xdd = \J\qdd + \Jd\qd,
\end{align}
analogous to the least squares formulation of classical mechanics under Gauss's Principle of Least Constraint \cite{udwadia1996analytical}. Because it's an optimality principle, the fundamental minimum does not change under changes of coordinates.

\subsection{Multivariate calculus notation}
\label{apx:CalculusNotation}

We avoid tensor notation here whenever possible in favor of the more compact and familiar matrix-vector notation of advanced calculus.
 
We assume each partial derivative generates a new index ranging over the partials. The partial $\partial_\x f(\x)$ of a function $f:\R^n\rightarrow\R$ is treated as a column vector. And if $\g:\R^n\rightarrow\R^m$ is a map, $\g(\x)$ itself is treated as a column vector and $\partial_\x \g(\x) \in \R^{m\times n}$ denotes a matrix with partials ranging over the columns. When there are multiple partials as in $\partial_{\x\y}h(\x,\y)$, we apply the partials in the order listed. In this case, we apply $\x$ first to create a column vector $\partial_\x h(\x,\y)$ then apply $\y$ to create a matrix $\partial_\y\big(\partial_\x h\big)$. 
When more partials are involved, the resulting object becomes a higher-order tensor, but we use the convention that multiplication on the right by a matrix or vector contracts the indices in the reverse order they were generated. For instance, if $b(\x,\y,\z)$ is a function, $\partial_{\x\y\z}b$ is a three index tensor, but $\partial_{\x\y\z}b \vv$ will unambiguously contract along the $\z$ partial first leaving a matrix whose rows align with the partials along $\x$ and whose columns align with the partials along $\y$. Specifically, $\partial_{\x\y\z}b \vv = \sum_k \frac{\partial}{\partial z^k}\big(\partial_{\x\y}b\big) v^k$. This convention is applied recursively when more than three variables are involved.

\section{Additional comments on related work}

\subsection{Generalization of Dynamic Movement Primitive (DMP) principles}
\label{sec:GeneralizingDmps}

In the Related Work section (see Section~\ref{sec:RelatedWork}), we discussed how Dynamic Movement Primitives (DMPs) are perturbations from a stable second-order differential equation and, in principle, can easily be generalized to leverage stable differential equations beyond simple linear equations. We describe that connection in more detail here. Studying these generalized perturbed systems is an area of future work, but we note here that geometric fabrics are not inherently competitive with the principles of DMPs, but rather complimentary.

A DMP is a nonlinear force perturbation away from an intrinsically stable linear second-order differential equation \cite{IjspeertDMPs2013}. When the force perturbation vanishes over time the system dynamics is asymptotically governed by the linear differential equation and therefore becomes a point attractor; likewise, when the nonlinear force is periodic the system converges to a stable limit cycle. In all cases, the key innovation is perturbing a stable differential equation using either a vanishing or cyclic perturbation so that stability of the combined system is governed by the stability of the underlying differential equation. 

These ideas can easily be generalized to geometric mechanics noting that inherently stable classical mechanical systems can be used as a drop-in replacement for the linear second-order differential equation to inherit stability. Since geometric fabrics are a generalization of those classical mechanical systems, the same applies for fabrics. Geometric fabrics can act as a stable underlying nonlinear second-order differential equation that we can perturb away from using a trained perturbation component generalizing the techniques studies in DMPs while gaining more expressive underlying system behaviors from this broader class of geometric fabric equations. 

\subsection{Geometric Dynamical Systems vs Finsler Systems}
\label{apx:Gds}

Geometric Dynamical Systems (GDSs) were introduced in \cite{cheng2018rmpflow} as a stable class of RMPs for stable system design. But we now understand GDSs are more elegantly and concretely expressed as Finsler systems characterized in Subsection~\ref{sec:FinslerAndHd2Geometries}. Moreover, such systems are inherently covariant since they derive from the Euler-Lagrange equations. This section presents that connection in detail.

We start by reviewing Geometric Dynamical Systems (GDS) and express them in a (tensor notion) form that makes it easier to see the relationship to Lagrangian and Finsler systems. Let $\G(\q,\qd)$ be a velocity dependent metric tensor and define an energy function as $\mathcal{K}(\q,\qd) = \frac{1}{2}\qd^\tr\G(\q,\qd)\qd$. Following the notation of \cite{cheng2018rmpflow}, we denote the columns of $\G$ by $\g_i$ so that $\G = [\g_1,\cdots,\g_d]$. An unforced GDS is defined as
\begin{align*}
    \M(\q,\qd) \qdd + \bm{\xi}(\q,\qd) = \zero.
\end{align*}
where $\M = \G + \bm{\Xi}$ has additional mass defined by $\bm{\Xi} = \frac{1}{2}\sum_i\dot{q}^i\partial_\qd\g_i$ (here $\dot{q}^i$ denotes the $i$th component of $\qd$) and $\bm{\xi} = \mathring{\G}\qd - \frac{1}{2}\partial_\q\big(\qd^\tr\G\qd\big)$ using $\mathring{\G} = [\h_1,\cdots,\h_d]$ with $\h_i = \big(\partial_\q\g_i\big)\qd$. When $\G$ is independent of $\qd$, these equations of motion reduce to the classical equations since $\M\rightarrow\G$ and $\mathring{\G} \rightarrow \dot{\G}$.

GDS systems conserve total energy $\mathcal{E} = \mathcal{K} + \psi(\q)$ when forced by a potential function $\psi(\q)$ and are shown to optimize $\psi$ when energy is bled off using a damping term $-\B(\q,\qd)\qd$ for positive definite $\B$ similar to the setting outlined in Theorem~\ref{thm:GeometricFabricStability}.

Using tensor notation with the Einstein summation convention, this system can be expressed $M_{kl}\ddot{q}^l + \xi_{kij}\dot{q}^i\dot{q}^j = 0$ where $M_{kl} = G_{kl} + \Xi_{kl}$ with
\begin{align} 
    \label{eqn:GdsEinstein1}
    \Xi_{kl} &= \frac{1}{2}\frac{\partial G_{ik}}{\partial \dot{q}^l}\dot{q}^i \\
    \label{eqn:GdsEinstein2}
    \xi_{kij} &= \frac{1}{2} \left(
                \frac{\partial G_{kj}}{\partial q^i}
                + \frac{\partial G_{ik}}{\partial q^j}
                - \frac{\partial G_{ij}}{\partial q^k}
            \right),
\end{align}
where $\xi_{kij}$ are the standard Christoffel symbols of the first kind \cite{LeeRiemannianManifolds97}. The only augmentation to the classical system is the term $\Xi$ which vanishes when $\G$ is independent of velocity.

A GDS by itself is not covariant to change of coordinates \cite{BylardBonalliEtAl2021}, but can be made covariant with the use of transform trees (see {\em structured GDS} in \cite{cheng2018rmpflow} and the discussion above in Appendix~\ref{apx:NoteOnCovariance}).
However, the next theorem shows that Lagrangian systems, and specifically Finsler systems derived from energies of the form $\Lag_e(\q,\qd) = \frac{1}{2}\qd^\tr\G(\q,\qd)\qd$, make a qualitatively similar, but inherently covariant, alternative to GDS systems.

\begin{theorem} \label{thm:LagrangianEquivalent}
    Let $\mathcal{L} = \frac{1}{2}G_{ij}\dot{q}^i\dot{q}^j$
    be a Finsler energy with $G_{ij}(q,\dot{q})$ a function of both position 
    and velocity. Then the Finsler system's equations of motion (defined by the Euler-Lagrange equation) can be expressed
    \begin{align}
        M_{kl}\ddot{q}^l + \Upsilon_{kij} \dot{q}^i\dot{q}^j
        = 0,
    \end{align}
    where $M_{kl} = G_{kl} + \Xi_{kl}$ with velocity curvature mass
    \begin{align}
        \Xi_{kl} = 
            \frac{\partial G_{lj}}{\partial \dot{q}^k}\dot{q}^j
            + \frac{\partial G_{kj}}{\partial \dot{q}^l}\dot{q}^j
            + \frac{1}{2}\frac{\partial^2 G_{ij}}{\partial \dot{q}^l\partial \dot{q}^k}\dot{q}^i\dot{q}^j,
    \end{align}
    and velocity-augmented Christoffel symbols (of the first kind) given by
    \begin{align}
        \Upsilon_{kij} = 
            \frac{1}{2} \left(
                \frac{\partial G_{kj}}{\partial q^i}
                + \frac{\partial G_{ik}}{\partial q^j}
                - \frac{\partial G_{ij}}{\partial q^k}
                + \frac{\partial^2 G_{ij}}{\partial q^l\partial \dot{q}^k}\dot{q}^l
            \right).
    \end{align}
\end{theorem}
\begin{proof}
    This proof amounts primarily to calculation and gathering together the
    terms appropriately to construct $M_{kl}$ and $\Upsilon_{kij}$.

    The Lagrangian is $\mathcal{L} = \frac{1}{2}G_{ij}\dot{q}^i\dot{q}^j$
    and the Euler-Lagrange equation defines 
    $\frac{d}{dt}\frac{\partial \mathcal{L}}{\partial \dot{q}^k} 
    - \frac{\partial\mathcal{L}}{\partial q^k} = 0$. That gives
    \begin{align*}
        \underbrace{\frac{d}{dt}\left[\frac{\partial}{\partial \dot{q}^k}\left(
            \frac{1}{2} G_{ij} \dot{q}^i\dot{q}^j
        \right)\right]}_{\mathrm{term 1}}
        - \frac{1}{2}\frac{\partial G_{ij}}{\partial q^k}\dot{q}^j\dot{q}^i
        = 0.
    \end{align*}
    Term 1 becomes
    \begin{align*}
    &\frac{d}{dt}\left[\frac{\partial}{\partial \dot{q}^k}\left(
        \frac{1}{2} G_{ij} \dot{q}^i\dot{q}^j
    \right)\right]
    = 
    \frac{d}{dt}\left[
        \frac{1}{2}\frac{\partial G_{ij}}{\partial \dot{q}^k} \dot{q}^i\dot{q}^j
        + G_{kj}\dot{q}^j
    \right] \\
    &\ \ \ = \frac{1}{2}\left(
                \frac{\partial^2G_{ij}}{\partial q^l\partial \dot{q}^k}\dot{q}^l
                + \frac{\partial^2G_{ij}}{\partial \dot{q}^l\partial \dot{q}^k}\ddot{q}^l
            \right)\dot{q}^i\dot{q}^j
            + \frac{1}{2}\frac{\partial G_{ij}}{\partial \dot{q}^k} \ddot{q}^i\dot{q}^j \\
        &\ \ \ \ \ + \frac{1}{2}\frac{\partial G_{ij}}{\partial \dot{q}^k} \dot{q}^i\ddot{q}^j
        + \left(\frac{\partial G_{kj}}{\partial q^l}\dot{q}^l + \frac{\partial G_{kj}}{\partial \dot{q}^l}\ddot{q}^l\right)\dot{q}^j
        + G_{kj} \ddot{q}^j.
    \end{align*}
    Putting this together with the other terms and grouping them as needed,
    we get
    \begin{align*}
        &\frac{d}{dt}\frac{\partial \mathcal{L}}{\partial \dot{q}^k} 
            - \frac{\partial\mathcal{L}}{\partial q^k}
        = \Bigg(
            G_{kl}
        + \frac{1}{2}\left[
                \frac{\partial G_{lj}}{\partial \dot{q}^k} \dot{q}^j
                + \frac{\partial G_{il}}{\partial \dot{q}^k} \dot{q}^i
            \right]
            + \frac{\partial G_{kj}}{\partial \dot{q}^l} \dot{q}^j \\
        &\ \ \ \ \ \ \ \ \ \ \ \ \ \ \ \ \ \ \ \ \ \ \ \ \ \ \ \,
        + \frac{1}{2}\frac{\partial^2 G_{ij}}{\partial \dot{q}^l\dot{q}^k}
                \dot{q}^i\dot{q}^j
            \Bigg) \ddot{q}^l \\
        &\ \ \ \ \ \ + \left(
            \frac{1}{2}\frac{\partial^2G_{ij}}{\partial q^l\partial \dot{q}^k}\dot{q}^l
            + \frac{\partial G_{kj}}{\partial q^i}
            - \frac{1}{2} \frac{\partial G_{ij}}{\partial q^k}
        \right) \dot{q}^i\dot{q}^j.
    \end{align*}
    Since $G_{ij}$ is symmetric, we have 
    $\frac{\partial G_{lj}}{\partial \dot{q}^k} \dot{q}^j 
    = \frac{\partial G_{il}}{\partial \dot{q}^k} \dot{q}^i$, and 
    \begin{align*}
        \frac{\partial G_{kj}}{\partial q^i}\dot{q}^i\dot{q}^j
        &= 
        \frac{1}{2}\frac{\partial G_{kj}}{\partial q^i}\dot{q}^i\dot{q}^j
        + \frac{1}{2}\frac{\partial G_{jk}}{\partial q^i}\dot{q}^i\dot{q}^j \\
        &= \frac{1}{2}\left(
            \frac{\partial G_{kj}}{\partial q^i}
            + \frac{\partial G_{ik}}{\partial q^j}
        \right)\dot{q}^i\dot{q}^j,
    \end{align*}
    so we can rewrite the above expression in its final form
    \begin{align*}
        &\frac{d}{dt}\frac{\partial \mathcal{L}}{\partial \dot{q}^k} 
            - \frac{\partial\mathcal{L}}{\partial q^k} \\
        &\ \ \ = \left(
            G_{kl} 
            + \frac{\partial G_{lj}}{\partial \dot{q}^k} \dot{q}^j
            + \frac{\partial G_{kj}}{\partial \dot{q}^l} \dot{q}^j
            + \frac{1}{2}\frac{\partial^2 G_{ij}}{\partial \dot{q}^l\dot{q}^k}
                \dot{q}^i\dot{q}^j
            \right) \ddot{q}^l \\
        &\ \ \ \ \ \ + \frac{1}{2}\left(
            \frac{\partial^2G_{ij}}{\partial q^l\partial \dot{q}^k}\dot{q}^l
            + \frac{\partial G_{kj}}{\partial q^i}
            + \frac{\partial G_{ik}}{\partial q^j}
            - \frac{\partial G_{ij}}{\partial q^k}
        \right) \dot{q}^i\dot{q}^j \\
        &\ \ = \Big(G_{kl} + \Xi_{kl}\Big)\ddot{q}^l
            + \Upsilon_{kij}\dot{q}^i\dot{q}^j,
    \end{align*}
    with $\Xi_{kl}$ and $\Upsilon_{kij}$ as given in the above theorem.
\end{proof}
Comparing the terms of Theorem~\ref{thm:LagrangianEquivalent} to Equations~\ref{eqn:GdsEinstein1} and \ref{eqn:GdsEinstein2} we see that Finsler systems of this form contains additional terms deriving from the velocity dependence in both the added mass and the augmented Christoffel symbols. Since these equations are derived using the Euler-Lagrange equation, they're known to be covariant to reparameterization; the added terms are critical for that covariance property to hold. Moreover, the class of all Finsler systems is a superset of these GDS-like Finsler systems, so Finsler systems broadly have many advantages over GDSs. 

Both GDSs and Finsler systems, however, generalize metrics to be velocity dependent but remain subject to the other limitations outlined in Subsection~\ref{sec:ShortcomingsOfClassicalMechanics}. Namely, their metrics and associated policies remain coupled. In practice, in the earlier results of \cite{cheng2018rmpflow}, high gains on forcing potentials and damping terms were used in order to leverage the systems for metric design while also reject the resulting irrelevant geometric terms that resulted. Geometric fabrics added a new dimension of flexibility by introducing bending terms to circumvent these issues.

\section{Additional Lemmas}

This Lemma gives the time derivative of a Lagrangian's energy (the Hamiltonian), which is a common calculation needed for multiple proofs below.

\begin{lemma}\label{lma:EnergyTimeDerivative}
Let $\Lag(\q,\qd)$ be any Lagrangian with Hamiltonian (energy) $\Ham_\Lag = \partial_\qd\Lag^\tr \qd - \Lag$. The energy time derivative is 
\begin{align} \label{eqn:EnergyTimeDerivative}
    \dot{\Ham}_\Lag = \qd^\tr\big(\M\qdd + \bm{\xi}\big),
\end{align}
where $\M$ and $\bm{\xi}$ come from the Lagrangian's equations of motion $\M(\q,\qd)\qdd + \bm{\xi}(\q,\qd) = \zero$.
\end{lemma}
\begin{proof}
The calculation is a straightforward time derivative of the Hamiltonian:
\begin{align*}
    \dot{\Ham}_{\Lag} 
    &= \frac{d}{dt}\Big[\partial_\qd\Lag^\tr\qd - \Lag\Big] \\
    &= \big(\partial^2_{\qd\qd}\Lag\,\qdd + \partial_{\qd\q}\Lag\,\qd\big)^\tr \qd + \partial_\qd\Lag^\tr\qdd \\
    &\ \ \ \ \ \ \ \ - \big(\partial_\qd\Lag^\tr\qdd + \partial_\q\Lag^\tr\qd \big) \\
    &= \qd^\tr\Big(\partial^2_{\qd\qd}\Lag\,\qdd + \partial_{\qd\q}\Lag\,\qd - \partial_\q\Lag\Big) \\
    &= \qd^\tr\big(\M\qdd + \bm{\xi}\big).
\end{align*}
where $\M = \partial^2_{\qd\qd}\Lag$ and $\bm{\xi} = \partial_{\qd\q}\Lag\,\qd - \partial_\q\Lag$.
\end{proof}

The following lemma collects some results around common matrices and operators that arise when analyzing energy conservation.
\begin{lemma}\label{lma:EnergyProjection}
Let $\M$ be a symmetric positive definite matrix and define $\p = \M\xd$. Then
\begin{align} \label{eqn:Rpe}
    \mR_{\p} = \M^{-1} - \frac{\xd\,\xd^\tr}{\xd^\tr\M\xd}
\end{align}
has null space spanned by $\p$ and
\begin{align}
    \mR_\xd = \M - \frac{\p\,\p^\tr}{\p^\tr\M^{-1}\p}
\end{align}
has null space spanned by $\xd$. These matrices are related by
$\mR_\xd = \M\mR_{\p}\M$ and the matrix $\M^{-1}\mR_\xd = \mR_{\p}\M = \mP_e$ is a projection operator of the form
\begin{align} \label{eqn:EnergyProjector}
    \mP_e = \M^{\,\frac{1}{2}}\Big[\I - \hat{\vv} \hat{\vv}^\tr\Big]\M^{-\frac{1}{2}}
\end{align}
where $\vv = \M^{\,\frac{1}{2}}\xd$ and $\hat{\vv} = \frac{\vv}{\|\vv\|}$ is the normalized vector. 
Moreover, $\xd^\tr\mP_e = \zero$.
\end{lemma}
\begin{proof}
Right multiplication of $\mR_{\p}$ by $\p$ gives:
\begin{align*}
    \mR_{\p} \p 
    &= \left(\M^{-1} - \frac{\xd\,\xd^\tr}{\xd^\tr\M\xd}\right) \M\xd \\
    &= \xd - \xd \left(\frac{\xd^\tr\M\xd}{\xd^\tr\M\xd}\right) = \zero,
\end{align*}
so $\p$ lies in the null space. Moreover, the null space is no larger since each matrix is formed by subtracting off a rank 1 term from a full rank matrix. 

The relation between $\mR_{\p}$ and $\mR_\xd$ can be shown algebraically
\begin{align*}
    \M\mR_{\p}\M 
    &= \M\left(\M^{-1} - \frac{\xd\,\xd^\tr}{\xd^\tr\M\xd}\right)\M \\
    &= \M - \frac{\M\xd\,\xd^\tr\M}{\xd^\tr\M\xd} \\
    &= \M - \frac{\p\,\p^\tr}{\p^\tr\M^{-1}\p} = \mR_\xd.
\end{align*}
Since $\M$ has full rank, $\mR_\xd$ has the same rank as $\mR_{\p}$ and its null space must be spanned by $\xd$ since $\mR_\xd\xd = \M\mR_{\p}\M\xd = \M\mR_{\p}\p = \zero$.

With a slight algebraic manipulation, we get
\begin{align*}
    \M\mR_{\p} 
    &= \M\left(\M^{-1} - \frac{\xd\,\xd^\tr}{\xd^\tr\M\xd}\right) \\
    &= \M^{\,\frac{1}{2}} \left(\I - \frac{\M^{\,\frac{1}{2}}\xd\,\xd^\tr\M^{\,\frac{1}{2}}}{\xd^\tr\M^{\,\frac{1}{2}}\M^{\,\frac{1}{2}}\xd}\right) \M^{-\frac{1}{2}} \\
    &= \M^{\,\frac{1}{2}} \left(\I - \frac{\vv\vv^\tr}{\vv^\tr\vv}\right) \M^{-\frac{1}{2}} \\
    &= \mP_e
\end{align*}
since $\frac{\vv\,\vv^\tr}{\vv^\tr\vv} = \hat{\vv}\hat{\vv}^\tr$. Moreover, 
\begin{align*}
    \mP_e\mP_e 
    &= 
    \M^{\,\frac{1}{2}} \left(\I - \hat{\vv}\hat{\vv}^\tr\right) \M^{-\frac{1}{2}}\\
    &\ \ \ \ \ \ \cdot
    \M^{\,\frac{1}{2}} \left(\I - \hat{\vv}\hat{\vv}^\tr\right) \M^{-\frac{1}{2}} \\
    &= \M^{\,\frac{1}{2}} \left(\I - \hat{\vv}\hat{\vv}^\tr\right) \left(\I - \hat{\vv}\hat{\vv}^\tr\right) \M^{-\frac{1}{2}} \\
    &= \M^{\,\frac{1}{2}} \left(\I - \hat{\vv}\hat{\vv}^\tr\right) \M^{-\frac{1}{2}} = \mP_e,
\end{align*}
since $\mP_\perp = \I - \hat{\vv}\hat{\vv}^\tr$ is an orthogonal projection operator. Therefore, $\mP_e^2 = \mP_e$ showing that it is a projection operator as well.

Finally, we have 
\begin{align*}
    \xd^\tr\mP_e 
    &= \xd^\tr\M\mR_{\p}
    = \xd^\tr\M\left(\M^{-1} - \frac{\xd\,\xd^\tr}{\xd^\tr\M\xd}\right)\\
    &= \left[\xd^\tr - \left(\frac{\xd^\tr\M\xd}{\xd^\tr\M\xd}\right)\xd^\tr\right]
    = \zero.
\end{align*}
\end{proof}

\section{Proofs of lemmas and theorems in the paper}
\label{apx:proofs}

\subsection{Lemma~\ref{lma:EnergyConservation} on energy conservation and dissipation}
\label{apx:LemmaEnergyConservation}
Proof of Lemma~\ref{lma:EnergyConservation} on energy conservation and dissipation of Finsler systems with zero-work modification terms. See Subsection~\ref{sec:BendingFinslerGeometries} for the full statement.
\begin{proof}
This energy is conserved if its time derivative is zero. 
Substituting $\qdd = -\M^{-1}\big(\bm{\xi} + \f_f + \partial\psi + \B\qd\big)$ into Equation~\ref{eqn:EnergyTimeDerivative} of Appendix Lemma~\ref{lma:EnergyTimeDerivative} using Lagrangian $\Lag(\q, \qd) = \Lag_e - \psi$
gives
\begin{align*}
    \dot{\Ham}_{\Lag_e} 
    &= \qd^\tr\Big(\M \big(-\M^{-1}(\bm{\xi} + \f_f + \partial\psi + \B\qd)\big)+ \bm{\xi} + \partial\psi\Big) \\
    &= \qd^\tr\big(-(\bm{\xi}+\partial\psi) 
        +(\bm{\xi} + \partial\psi) 
        - (\f_f + \B\qd)\big) \\
    &= -\qd^\tr\f_f - \qd^\tr\B\qd.
\end{align*}
Therefore, energy reduces at a rate $-\qd^\tr\B\qd$ (and is conserved when $\B=\zero$) if and only if $\f_f$ is zero-work with $\qd^\tr\f_f = 0$ everywhere. 

Finally, when $\f_f$ is geometric $\M\qdd + \bm{\xi} + \f_f = \zero$ is geometric since $\M$ is HD0 and $\bm{\xi}$ is HD2. Moreover, since $\Lag_e$ is lower bounded by 0, $\Ham=\Lag_e+\psi$ is lower bounded when $\psi$ is.
\end{proof}

\subsection{Lemma~\ref{lma:ConstrainedGeometricFabrics} on constrained fabrics}
\label{apx:LemmaConstrainedGeometricFabrics}

Proof of Lemma~\ref{lma:ConstrainedGeometricFabrics} in  Subsection~\ref{sec:GeometricFabrics} showing that constrained fabrics are also fabrics.
\begin{proof}
The formula derives from the standard algebra of constrained systems used in classical mechanics \cite{ClassicalMechanicsTaylor05}, which we derive here for completeness. The Principle of Virtual Work tells us the constraint force takes the form $\f_c = \J^\tr\lm$ for unknown Lagrange multipliers $\lm$. (Since $\J$'s rows are orthogonal to the constraint's tangent space, $\f_c$ defined this way is always zero work for any valid velocity consistent with the constraints.) The constraint $\C(\q) = \zero$ induces two additional constraints by its time derivatives $\dot{\C} = \J\qd = \zero$ and $\ddot{\C} = \J\qdd + \Jd\qd = \zero$. The constraint and its first derivative are satisfied by assumption, so it's the third that enters into the system:
\begin{align*}
    &\M\qdd + \f + \J^\tr\lm = \zero\\
    &\J\qdd + \Jd\qd = \zero,
\end{align*}
where we've captured all force terms in $\f = \bm{\xi} + \f_f + \partial\psi + \B\qd$. Multiplying the first equation by $\J\M^{-1}$ and subtracting the second, we can solve for $\lm$:
\begin{align*}
    \lm = (\J\M^{-1}\J^\tr)^{-1}\big[-\J\M^{-1}\f+\Jd\qd\big].
\end{align*}
These Lagrange multipliers define the constraint forces; plugging them back into the first equation and collecting terms gives:
\begin{align*}
    \M\qdd + \mP_\paral\f + \mP_\perp\M\Jd\qd = \zero.
\end{align*}
This result gives Equation~\ref{eqn:ConstrainedGeometricFabrics} once $\f$ is split apart. Additionally, adding and subtracting $\bm{\xi}$ gives
\begin{align*}
    &\M\qdd + \bm{\xi} - \bm{\xi} + \mP_\paral\big(\bm{\xi}+\f_f\big) + \mP_\perp \M\Jd\qd \\
    &\ \ \  
    = \M\qdd + \bm{\xi} + \mP_\perp\big(\M\Jd\qd - \bm{\xi}\big) + \mP_\paral\f_f,
\end{align*}
which derives the expression for $\wt{\f}_f = \mP_\perp\big(\M\Jd\qd - \bm{\xi}\big) + \mP_\paral\f_f$

Observing that $\M$ is HD0 (see discussion in Section~\ref{sec:FinslerAndHd2Geometries}), we see that both $\mP_\paral$ and $\mP_\perp$ are HD0 (i.e. they do not depend on $\|\qd\|$). Since $\f_f$ is HD2 by definition, its projection is as well. Moreover, $\bm{\xi}$ is known to be HD2 and using tensor notation, we see $\Jd\qd = \frac{\partial^2\C}{\partial q^i\partial q^j}\dot{q}^i\dot{q}^j$ is HD2, which means means both $\M\Jd\qd - \bm{\xi}$ and its projection are HD2 as well. Together these results show that $\wt{\f}_f$ is HD2, and since we know it derives from the original bending term by adding a zero-work constraint force, $\wt{\f}_f$ must be zero work as well and is hence a bending term. This proves that $(\Lag_e, \wt{\f}_f)$ is a geometric fabric.
\end{proof}

\subsection{Theorem~\ref{thm:GeometricFabricStability} on Geometric Fabric Stability}
\label{apx:TheoremGeometricFabricStability}

Proof of theorem in Subsection~\ref{sec:StabilityAnalysis} characterizing the stability of geometric fabrics and convergence to KKT solutions of the constrained potential.
\begin{proof}
By Lemma~\ref{lma:EnergyConservation}, the forced system will reduce total energy $\Ham(\q, \qd) = \Lag_e(\q, \qd) + \psi(\q)$, at a rate $\dot{\Ham} \leq -\qd^\tr\B(\q,\qd)\qd$ (the inequality is due to the impact model). Since $\Lag_e\geq 0$ and $\psi(\q)$ is lower bounded, $\Ham$ is lower bounded and it must be that $\dot{\Ham}\rightarrow 0$. Therefore, $\qd^\tr\B(\q,\qd)\qd\rightarrow 0$ which implies $\qd\rightarrow\zero$ since $\B$ is positive definite. Since velocity approaches zero we must also have $\qdd\rightarrow\zero$. At convergence, a subset of inequality constraints will be active resulting in an equality constraint for those dimensions which we denote $\wt{\g} = \zero$. Combining those active inequality constraints with the equality constraints $\mc$ into $\C = [\mc;\wt{\g}]$ with Jacobian $\J = \partial_\q\C$, the constrained the equations of motion at convergence can be expressed
\begin{align}
    \M\qdd + \bm{\xi} + \f_f + \J^\tr\lm= -\partial_\q\psi - \B\qd.
\end{align}
Since $\qdd\rightarrow\zero$ we have $\M\qdd\rightarrow\zero$ ($\M$ is HD0) and since $\qd\rightarrow\zero$ we have $\bm{\xi}+\f_f \rightarrow\zero$ (both $\bm{\xi}$ and $\f_f$ are HD2). That leaves just $\J^\tr\lm$ on the left hand side. On the right hand side we have $\B\qd\rightarrow\zero$ since $\B$ is bounded, so the system approaches
\begin{align}
    \J^\tr\lm = -\partial_\q\psi.
\end{align}
Therefore, it approaches a KKT solution since the system evolution constraints are everywhere satisfied and contact models require $\lm$ satisfy the required complimentarity conditions. 
\end{proof}

\subsection{Theorem~\ref{thm:SystemEnergization} on Energized Fabrics}
\label{apx:TheoremOnEnergizedFabrics}

Proof of Theorem~\ref{thm:SystemEnergization} in Subsection~\ref{sec:EnergizingHd2Geometries} showing how to energize arbitrary HD2 geometries and proving the resulting system is a geometric fabric.
\begin{proof}
We derive these results using notation $\h = -\pi(\q, \qd)$ so that the HD2 geometry can be expressed $\qdd + \h(\q,\qd) = \zero$. Using $\M \qdd + \bm{\xi} = \zero$ to denote the unforced equations of motion,
Equation~\ref{eqn:EnergyTimeDerivative} of Appendix Lemma~\ref{lma:EnergyTimeDerivative} gives the time derivative of the energy $\dot{\Ham}_e = \qd^\tr\big(\M\qdd + \bm{\xi}\big)$. Substituting a system of the form $\qdd = -\h(\q,\qd) - \alpha_{\Ham_e}\qd$, setting it to zero, and solving for $\alpha_{\Ham_e}$ gives
\begin{align*}
&\dot{\Ham}_e = \qd^\tr\Big[\M\big(-\h - \alpha_{\Ham_e}\qd\big) + \bm{\xi}\Big] = \zero \\
&\Rightarrow -\qd^\tr\M\h - \alpha_{\Ham_e} \qd^\tr\M\qd + \qd^\tr\bm{\xi} = \zero \\
&\Rightarrow \alpha_{\Ham_e} = -\frac{\qd^\tr\M\h - \qd^\tr\bm{\xi}}{\qd^\tr\M\qd}\\
&\ \ \ \ \ \ \ \ \ \ = 
-(\qd^\tr\M\qd)^{-1}\qd^\tr\big[\M\h - \bm{\xi}\big].
\end{align*}
This result gives the formula for Equation~\ref{eqn:EnergizationTransformAlpha}. Substituting this solution for $\alpha_{\Ham_e}$ back in gives
\begin{align*}
    \qdd 
    &= -\h + \left(\frac{\qd^\tr}{\qd^\tr\M\qd}\big[\M\h - \bm{\xi}\big]\right) \qd \\
    &= -\h + \left[\frac{\qd\,\qd^\tr}{\qd^\tr\M\qd}\right]\big(\M\h - \bm{\xi}\big).
\end{align*}
Algebraically, it helps to introduce $\h = \M^{-1}\big[\f_f + \bm{\xi}\big]$ to first express the result as a difference away from $\bm{\xi}$; in the end we will convert back to $\h$. Doing so and moving all the terms to the left hand side of the equation gives
\begin{align*}
    &\ \ \ \ \qdd + \h - \left[\frac{\qd\,\qd^\tr}{\qd^\tr\M\qd}\right]\big(\M\h - \bm{\xi}\big) = \zero \\
    &\Rightarrow \qdd + \M^{-1}\big[\f_f + \bm{\xi}\big] \\
    &\ \ \ \ \ - \left[\frac{\qd\,\qd^\tr}{\qd^\tr\M\qd}\right]\Big(\M\M^{-1}\big[\f_f + \bm{\xi}\big] - \bm{\xi}\Big) = \zero \\
    &\Rightarrow \M \qdd + \f_f + \bm{\xi} \\
    &\ \ \ \ \ - \M\left[\frac{\qd\,\qd^\tr}{\qd^\tr\M\qd}\right]\Big(\f_f + \big(\bm{\xi} - \bm{\xi}\big)\Big) = \zero \\
    &\Rightarrow \M \qdd + \bm{\xi} + \M\left[\M^{-1} - \frac{\qd\,\qd^\tr}{\qd^\tr\M\qd}\right] \f_f = \zero \\
    &\Rightarrow \M \qdd + \bm{\xi} + \M\mR_{\p}\big(\M\h - \bm{\xi}\big) = \zero,
\end{align*}
where $\mR_{\p} = \M^{-1} - \frac{\qd\,\qd^\tr}{\qd^\tr\M\qd}$ matching Equation~\ref{eqn:Rpe} of Lemma~\ref{lma:EnergyProjection}, and we substitute $\f_f = \M\h - \bm{\xi}$ back in. By Lemma~\ref{lma:EnergyProjection} $\M\mR_{\p} = \mP_e$, so we get Equation~\ref{eqn:ZeroWorkEnergizationForm}.
Moreover, again by Lemma~\ref{lma:EnergyProjection}, $\qd^\tr\mP_e = \zero$, so $\qd^\tr\f_f = -\qd^\tr\mP_e\big[\M\pi + \bm{\xi}\big] = 0$ showing that the new term is a zero work modification. Finally, both $\h$ and $\bm{\xi}$ are HD2 by definition, and since $\mP_e$ and $\M$ are HD0, the modifying term is HD2 making it a bending term. Therefore, the energized system is a bent Finsler system and hence a geometric fabric.
\end{proof}

\subsection{Theorem~\ref{thm:FabricsInGeneralizedCoords} on Fabrics in Generalized Coordinates}
\label{apx:TheoremFabricsInGeneralizedCoords}

Proof of Theorem~\ref{thm:FabricsInGeneralizedCoords} showing how to express in generalized coordinates energized systems that are subsequently constrained.
See Subsection~\ref{sec:GeneralizedCoordinates}.

\begin{proof}
We will show the equivalence by calculation. Let $\h(\x,\xd) = -\pi(\x,\xd)$ so that the geometry's equation becomes $\xdd + \h(\x,\xd) = \zero$. Multiplying through by $\M$ gives its metric-weighted (force) form $\M\xdd + \f = \zero$ where $\f = \M\h$.
Using Equation~\ref{eqn:EnergizationTransformAlpha} of Theorem~\ref{thm:SystemEnergization}, this system's energization with respect to energy $\Lag_e$ with equations of motion $\M\xdd + \bm{\xi} = \zero$ can be expressed
$\M\xdd + \bm{\xi}^{\h} = \zero$ where
\begin{align}
    \bm{\xi}^{\h} = \bm{\xi} + \M\left[\M^{-1} - \frac{\xd\xd^\tr}{\xd^\tr\M\xd}\right] \big(\f - \bm{\xi}\big).
\end{align}
Let $\J = \partial_\q\phi$. The pullback of the energized geometry generator is
\begin{align}\nonumber
    &\J^\tr\M\left(\J\qdd + \Jd\qd\right) + \J^\tr\bm{\xi}^{\h} = \zero \\\nonumber
    &\Rightarrow \big(\J^\tr\M\J\big)\qdd + \J^\tr\big(\bm{\xi}^{\h} + \M\Jd\qd\big) = \zero\\\nonumber
    &\Rightarrow \big(\J^\tr\M\J\big)\qdd + \J^\tr \bm{\xi} \\\nonumber
    &\ \ \ \ + \J^\tr\M\left[\M^{-1} - \frac{\xd\xd^\tr}{\xd^\tr\M\xd}\right]
            \big(\f - \bm{\xi}\big) 
        + \J^\tr\M\Jd\qd = \zero \\\label{eqn:EnergizationPullback}
    &\Rightarrow \wt{\M}\qdd + \wt{\bm{\xi}} 
        + \J^\tr\M\left[\M^{-1} - \frac{\xd\xd^\tr}{\xd^\tr\M\xd}\right]
            \big(\f - \bm{\xi}\big),
\end{align}
where $\wt{\M} = \J^\tr\M\J$ and $\wt{\bm{\xi}} = \J^\tr\big(\bm{\xi} + \M\Jd\qd\big)$ form the standard pullback of $(\M, \bm{\xi})$.

Now we show that pulling back the geometry first then energizing it in $\q$ results in a system matching Equation~\ref{eqn:EnergizationPullback}. We can calculate the geometry pullback with respect to the energy metric $\M$ by pulling back the metric weighted force form of the geometry $\M\xdd + \f = \zero$, where again $\f = \M\h$. The pullback is
\begin{align}\nonumber
    &\J^\tr\M\big(\J\qdd + \Jd\qd\big) + \J^\tr\f = \zero \\\nonumber
    &\Rightarrow \big(\J^\tr\M\J\big) \qdd + \J^\tr\big(\f + \M\Jd\qd\big) \\\label{eqn:PullbackGeometry}
    &\Leftrightarrow \wt{\M}\qdd + \wt{\f} = \zero
\end{align}
where $\wt{\M} = \J^\tr\M\J$ as before and $\wt{\f} = \J^\tr \big(\f + \M\Jd\qd\big)$.

Let $\wt{\Lag}_e = \Lag_e\big(\phi(\q), \J\qd\big)$ be the pullback of the energy function $\Lag_e$. We know that the Euler-Lagrange equation commutes with the pullback, so applying the Euler-Lagrange equation to this pullback energy $\wt{\Lag}_e$ is equivalent to pulling back the Euler-Lagrange equation of $\Lag_e$. This means we can calculate the Euler-Lagrange equation of $\Lag_e$ as
\begin{align}
    &\big(\J^\tr\M\J\big) \qdd + \J^\tr\big(\bm{\xi} + \M\Jd\qd\big) = \zero\\
    &\Leftrightarrow \wt{\M}\qdd + \wt{\bm{\xi}} = \zero,
\end{align}
with $\wt{\M} = \J^\tr\M\J$ and $\wt{\bm{\xi}} = \J^\tr\big(\bm{\xi} + \M\Jd\qd\big)$ (both as previously defined). Therefore, energizing \ref{eqn:PullbackGeometry} with $\wt{\Lag}_e$ gives
\begin{align}\nonumber
    &\wt{\M}\qdd + \wt{\bm{\xi}} 
      + \wt{\M}\left[
        \wt{\M}^{-1} - \frac{\qd\qd^\tr}{\qd^T\wt{\M}\qd}\right]
        \big(\wt{\f} - \wt{\bm{\xi}}\big) = \zero \\\nonumber
    &\Rightarrow 
    \wt{\M}\qdd + \wt{\bm{\xi}} \\\nonumber
    &\ \ \ \ + \big(\J^\tr\M\J\big)
          \left[
            \wt{\M}^{-1} - \frac{\qd\qd^\tr}{\qd^\tr \J^\tr\M\J\qd}
          \right] \cdot\\\nonumber
    &\ \ \ \ \ \ \ \ \ \ \ \ \ \ \ \ \ 
        \Big(\J^\tr\big(\f + \M\Jd\qd\big)
          - \J^\tr\big(\bm{\xi} + \M\Jd\qd\big)
        \Big)
      = \zero \\\nonumber
    &\Rightarrow
    \wt{\M}\qdd + \wt{\bm{\xi}} \\\nonumber
    &\ \ \ \ + \big(\J^\tr\M\big)
          \J\left[
            \big(\J^\tr\M\J\big)^{-1} - \frac{\qd\qd^\tr}{\xd^\tr\M\xd}
          \right]\J^\tr
        \big(\f- \bm{\xi}\big) = \zero\\\label{eqn:PullbackEnergizationAlmost}
    &\Rightarrow
    \wt{\M}\qdd + \wt{\bm{\xi}} \\\nonumber
    &\ \ \ \ + \J^\tr\M
          \left[
            \J\big(\J^\tr\M\J\big)^{-1}\J^\tr - \frac{\xd\xd^\tr}{\xd^\tr\M\xd}
          \right]
        \big(\f- \bm{\xi}\big) = \zero.
\end{align}
Since $\phi$ is full rank $\J^\tr\M\J$ is invertible, so
\begin{align*}
    \J^\tr\M\J\big(\J^\tr\M\J\big)^{-1}\J^\tr
    = \J^\tr = \J^\tr\M\big(\M^{-1}\big),
\end{align*}
which means we can rewrite Equation~\ref{eqn:PullbackEnergizationAlmost} as
\begin{align}\label{eqn:PullbackEnergized}
    \wt{\M}\qdd + \wt{\bm{\xi}}
      + \J^\tr\M
          \left[
            \M^{-1} - \frac{\xd\xd^\tr}{\xd^\tr\M\xd}
          \right]
        \big(\f- \bm{\xi}\big) = \zero.
\end{align}
This final expression in Equation~\ref{eqn:PullbackEnergized} matches the expression for the energized geometry pullback in Equation~\ref{eqn:EnergizationPullback}.
\end{proof}

\section{Details on Designing with Fabrics}

\subsection{Derivations of Fabric Component Algebra}
\label{apx:DerivationsOfFabricComponentAlgebra}

A fabric component was defined in Subsection~\ref{sec:AlgebraOfComponents} as a Finsler energy paired with a force term $\{\Lag_e, \f\}$. We often refer to this pair as a Finsler energy paired with a {\em HD2 geometry} $\pi(\x,\xd)$ knowing that the actual component would be $\f = -\M\pi$ where $\M$ is the mass from the energy's equations of motion $\M\xdd + \bm{\xi} = \zero$. Since the energization formulas in Theorem~\ref{thm:SystemEnergization} are functions of the Finsler energy $\Lag_e$ only through its equations of motion parameterized by $(\M,\bm{\xi})$, we can think of the fabric component as effectively a triple $(\M, \bm{\xi},\f)$ representing the energy equations alongside an additional forcing term, which itself can be viewed as two pairs $\big((\M, \bm{\xi}), (\M,\f)\big)$. The algebra defined in Subsection~\ref{sec:AlgebraOfComponents} is equivalent to the reduction of sums and pullback to the sums and pullback of these individual pairs. For concreteness, and in alignment with the terminology of \cite{optimizationFabricsForBehavioralDesignArXiv2020}, we call these pairs {\em specs}, short for {\em spectral semi-sprays}, which is descriptive of the differential equation they represent emphasizing the spectral role of the metric $\M$.

The one operation on these specs not discussed in Subsection~\ref{sec:AlgebraOfComponents} is summation and the linearity of summation. This property can be easily seen by
\begin{align*}
    &\alpha(\M_1,\f_1) + \beta(\M_2,\f_2) \\
    &\ \ \Leftrightarrow
    \alpha \big(\M_1\xdd + \f_1\big) + \beta \big(\M_2\xdd + \f_2\big) = \zero \\
    &\Rightarrow
    \big(\alpha\M_1 + \beta\M_2\big)\xdd + \big(\alpha\f_1 + \beta\f_2\big) = \zero \\
    &\ \ \Leftrightarrow \Big(\big(\alpha\M_1 + \beta\M_2\big), \big(\alpha\f_1 + \beta\f_2\big)\Big).
\end{align*}
As mentioned in the discussion leading up to Equation~\ref{eqn:PullbackCommutesWithEL}, the pullback of energy equations is equivalent to applying the Euler-Lagrange equation on the pullback energy. Since application of the Euler-Lagrange equation is linear in the sense $\mathrm{EL}\big[\alpha \Lag_1 + \beta \Lag_2\big] = \alpha\,\mathrm{EL}\big[\Lag_1\big] + \beta\,\mathrm{EL}\big[\Lag_2\big]$, 
if we treat any two energies that differ by only a constant as equivalent, the mapping defined by the $\mathrm{EL}[\cdot]$ operator defines an isomorphism between energies and the subspace of specs defined by those energies. That means, algebraically, we can treat the two representations as equivalent.

That algebra of specs derives the summation operation on fabric components. For the pullback operation, we can simply note that by Equation~\ref{eqn:PullbackCommutesWithEL}, the first (energy) spec of the component has a pullback matching the application of the Euler-Lagrange equation applied to the pullback energy, so the pullback of the energy captures that part. It additionally captures the pullback of the mass matrix of the second spec, so we need only additionally account for how that second spec's force term pulls back. Looking at the definition of component pullback in Equation~\ref{eqn:SpecPullback} we see that that's exactly what the force term pullback expression computes.

Finally, we note that while the definition of fabric components and its associated algebra as presented in Subsection~\ref{sec:AlgebraOfComponents} is compact and gives insight into what fabric components are and how they interact, it may be more efficient in practice to represent them in compact spec form explicitly as triples $(\M,\bm{\xi},\f)$ in code. That said, leveraging automatic differentiation frameworks such as those used in \cite{LiRmp2RobotLearning2021} may lead to alternative efficient representations. 

We note that some fabric components may represent potential terms with dampers as their $\f$ term. These components intermix with the other fabric components using the same algebra since a component is defined independent of $\f$ semantics.

\subsection{Formulas for speed control}
\label{apx:FormulasSpeedControl}

This section gives explicit formulas for calculating the damping $\beta(\q,\qd)$ described in Subsection~\ref{sec:SpeedControl} for regulating a specific measure of speed.

Once a geometry is energized by a Finsler energy $\Lag_e$, its speed profile is defined by what is necessary to conserve $\Lag_e$. In practice, we usually want to regulate a different measure of energy (speed), an {\em execution energy} $\Lag_e^{\mathrm{ex}}$, such as Euclidean energy in either C-space or end-effector space (or, more commonly, some combination of both). 
Denoting the fabric by $(\Lag_e, -\M\pi - \bm{\xi})$ which defines nominal fabric equations
\begin{align*}
    &\M \qdd + \bm{\xi} + \f_f \\
    &\ \ = \M \qdd + \bm{\xi} - \M\pi - \bm{\xi} \\
    &\ \ = \M (\qdd - \pi) = \zero
\end{align*}
so that the resolved behavior matches the geometric policy $\qdd = \pi$,
we can regulate the speed by adjusting $\alpha_\mathrm{reg}$ in a forced and damped equation of the form
\begin{align} \label{eqn:FundamentalsOfSpeedControl}
    \qdd = -\M^{-1}\partial_\q\psi(\q) + \pi(\q, \qd) + \alpha_\mathrm{reg} \qd.
\end{align}
As long as $\alpha_\mathrm{reg} < \alpha_{\Lag_e}$, where $\alpha_{\Lag_e}$ is the energization coefficient (Equation~\ref{eqn:EnergizationTransformAlpha} of Theorem~\ref{thm:SystemEnergization}), the system will remain stable per Theorem~\ref{thm:GeometricFabricStability} since $\alpha_\mathrm{reg}$ always act as a nonzero damper with respect to the energized system. Note that $\alpha_{\Lag_e}$ need not always be smaller than $0$, so we have the flexibility with $\alpha_\mathrm{reg}$ to both remove and add energy as needed (although, we don't necessarily have to use that flexibility; the method outlined below restricts it to being positive to aid intuition). Below we calculate explicit formulas for $\alpha_\mathrm{reg}$ to regulate a given measure of speed (energy).

Let $\Lag_e^\mathrm{ex}$ be an {\em execution} energy, which may differ from the fabric's energy $\Lag_e$, and let $\alpha_\mathrm{ex}^0$ and $\alpha_\mathrm{ex}^\psi$ be the energization coefficients, respectively, for $\mathrm{energize}_{\Lag_e^\mathrm{ex}}\big[\pi\big]$ and $\mathrm{energize}_{\Lag_e^\mathrm{ex}}\big[-\M^{-1}\partial_\q\psi + \pi\big]$. The coefficient $\alpha_\mathrm{ex}^0$ ensures that system 
\begin{align}
    \qdd = \pi(\q,\qd) + \alpha_\mathrm{ex}^0\qd
\end{align}
conserves $\Lag_e^{\mathrm{ex}}$ and the coefficient $\alpha_\mathrm{ex}^\psi$ ensures that the {\em forced} system\footnote{As noted above, the energization coefficient calculation in Theorem~\ref{thm:SystemEnergization} does not require that the underlying differential equation being energized be a geometry (in this case it's a geometry forced by a potential). It simply states that if the underlying differential equation is indeed a geometry, then the resulting energized equation is a geometric fabric. Here we're just using the result to calculate a reference coefficient.} 
\begin{align}
    \qdd = -\M^{-1}\partial_\q\psi + \pi(\q,\qd) + \alpha_\mathrm{ex}^\psi\qd
\end{align}
conserves $\Lag_e^{\mathrm{ex}}$.
Likewise, let $\alpha_{\Lag_e}$ denote the energy coefficient of the actual system energy $\mathrm{energize}_{\Lag_e}\big[\pi\big]$ (designed to conserve $\Lag_e$). 

The forced and damped energized system  
\begin{align}
    \qdd = -\M^{-1}\partial_\q\psi + \pi(\q,\qd) + \alpha_{\Lag_e}\qd - \wt{\beta} \qd,
\end{align}
would be stable for $\wt{\beta}>0$. Adding and subtracting $\alpha_\mathrm{ex}^0$ gives
\begin{align*}
    \qdd 
    &= -\M^{-1}\partial_\q\psi + \pi(\q,\qd) + \big(\alpha_{\Lag_e} + \alpha_\mathrm{ex}^0 - \alpha_\mathrm{ex}^0\big)\qd - \wt{\beta} \qd \\
    &= -\M^{-1}\partial_\q\psi + \pi(\q,\qd) + \alpha_\mathrm{ex}^0\qd - \beta \qd\\
    &= -\M^{-1}\partial_\q\psi + \mathrm{energize}_{\Lag_\mathrm{ex}}\big[\pi(\q,\qd)\big] - \beta \qd,
\end{align*}
where $\beta = \wt{\beta} + \alpha_\mathrm{ex}^0 - \alpha_{\Lag_e}$. The constraint $\wt{\beta}>0$ written this way manifests as $\beta > \alpha_\mathrm{ex}^0 - \alpha_{\Lag_e}$. 

This calculation suggests we can regulate speed by choosing $\alpha_\mathrm{reg}$ in Equation~\ref{eqn:FundamentalsOfSpeedControl} as
\begin{align} \label{eqn:AlphaRegForm}
    \alpha_\mathrm{reg} 
    = \alpha_\mathrm{ex}^\eta - \beta_\mathrm{reg}(\q, \qd) + \alpha_\mathrm{boost}
\end{align}
where $\alpha_\mathrm{ex}^\eta = \eta \alpha_\mathrm{ex}^0 + (1-\eta)\alpha_\mathrm{ex}^\psi$ for $\eta\in[0,1]$ and $\alpha_\mathrm{boost} \leq 0$ as long as $\beta_\mathrm{reg} > \alpha_\mathrm{ex}^0 - \alpha_{\Lag_e}$. 
The term $\alpha_\mathrm{boost}$ can be used temporarily in the beginning to boost the system up to speed. Since it is transient, we drop $\alpha_\mathrm{boost}$ momentarily for simplicity, but return to it at the end of the section to describe a good boosting policy. In our experiments, we use that term only for the cubby navigation problem in order to reduce the strength of the potential function near the beginnings of motions to get a clear signal on what the geometry itself is encoding.

Under this choice of $\alpha_\mathrm{reg}$ given in Equation~\ref{eqn:AlphaRegForm}, using $\vv_\paral = \big(\alpha_\mathrm{ex}^\psi - \alpha_\mathrm{ex}^0\big)\qd$ to denote the component of $-\M^{-1}\partial_\q\psi$ additionally removed by including it within the energization operation (and dropping $\alpha_\mathrm{boost}$), we can express the system as
\begin{align*}
    \qdd &= \eta \vv_\paral + \mathrm{energize}_{\Lag_e^\mathrm{ex}}\big[-\M^{-1}\partial_\q\psi + \pi\big] - \beta_\mathrm{reg}(\q, \qd)\qd,
\end{align*}
using a similar analysis to the above. Intuitively, the energized system will conserve execution energy, $\eta \vv_\paral$ can inject energy into the system to speed it up (by an amount parameterized by $\eta\in[0,1]$), and $\beta_\mathrm{reg}$ can bleed off energy. Note that the lower bound on the damper isn't strictly zero $\beta_\mathrm{reg} > \alpha_\mathrm{ex}^0 - \alpha_{\Lag_e}$, so there may be both cases where that term is allowed to inject extra energy and when that term is forced to be strictly a margin above zero for stability.

In practice, we enforce that $\beta_\mathrm{reg}$ is strictly positive by choosing $\beta_\mathrm{reg} = s_\beta(\q) B + \underline{B} + \max\{0, \alpha_\mathrm{ex}^\eta - \alpha_{\Lag_e}\} > 0$ with both $B$ and $\underline{B}$ positive. This gives us a more intuitive interface to speed control with $\eta$ adjusting energy in and $\beta_\mathrm{reg}$ adjusting energy out.
In this form, $\underline{B}$ acts as a constant baseline damping coefficient, and $s_\beta(\q) B$ provides additional damping near the convergence point with $B>0$ and $s_\beta(\q)$ acting as a switch transitioning from 0 to 1 as the system approaches the target. $\max\{0, \alpha_\mathrm{ex}^\eta - \alpha_{\Lag_e}\}$ ensures the stability bound $\alpha_\mathrm{reg} = \beta_\mathrm{reg} \geq \alpha_\mathrm{ex} - \alpha_{\Lag_e}$ is satisfied with the baseline coefficients making it a strict inequality.

Specifically, in our experiments we use
\begin{align}
\label{eqn:damping_gate}
 s_\beta(\q) = \frac{1}{2} \Big(\tanh\big(-\alpha_\beta (\|\q\| - r) \big) + 1\Big) 
\end{align}
where $\alpha_\beta \in \mathbb{R}^+$ is a gain defining the switching rate, and $r \in \mathbb{R}^+$ is the radius where the switch is half-way engaged. Denoting the desired execution energy as $\Lag_e^\mathrm{ex,d}$, we use the following policy for $\eta$
\begin{align}
\label{eqn:energy_gate}
    \eta = \frac{1}{2} \Big(\tanh\big(-\alpha_\eta (\Lag_e^\mathrm{ex} - \Lag_e^\mathrm{ex,d}) - \alpha_\mathrm{shift} \big) + 1\Big) 
\end{align}
where $\alpha_\eta,\alpha_\mathrm{shift} \in \mathbb{R}^+$ adjust the rate and offset, respectively, of the switch as an affine function of the speed (execution energy) error. 

When used, $\alpha_\mathrm{boost}$ acts to explicitly inject energy along the direction of motion to quickly boost the system up to speed. When $\qd = \zero$, this term has no effect, so the initial direction of motion is chosen by the potential $\-\partial_\q\psi$. Once $\qd \neq \zero$, $\alpha_\mathrm{boost}$ quickly accelerates the system to the desired speed. When $\pi$ is a geometry, quickly reaching a high speed means that the influence of the non-geometric potential is diminished, promoting path consistency. We reduce $\alpha_\mathrm{boost}$ to zero permanently once the system is close to convergence so it does not affect convergence stability.

In these cases, $\alpha_\mathrm{boost}$ is modeled as $\alpha_\mathrm{boost} = k \: \eta \big(1-s_\beta(\q)\big)\frac{1}{\|\qd\| + \epsilon}$, where $k \in \mathbb{R}^+$ is a gain that directly sets the desired level of acceleration, $\eta$ (from above) sets $\alpha_\mathrm{boost}=0$ when the desired speed is achieved, and $1-s_\beta(\q)$ sets $\alpha_\mathrm{boost}=0$ when the system is within the region of higher damping. The normalization by $\|\qd\|+\epsilon$ ensures that $\alpha_\mathrm{boost}$ is directly applied along $\hat{\qd}$ with a very small positive value for $\epsilon$ to ensure numerical stability. This overall design injects more energy into the system when $-\alpha_\mathrm{boost} < \alpha_\mathrm{ex}^\eta - \alpha_{\Lag_e}$. Since this injection occurs for finite time, the total system energy is still bounded. The additional switches ensure that the system is still subject to positive damping, guaranteeing convergence.

\subsection{Designing fabric components}
\label{apx:DesigningHD2Terms}

This section reviews a number of techniques for designing HD2 functions and maps for the design of Finsler energies and HD2 geometries referenced in Subsection~\ref{sec:ToolsForFabricDesign} on tools for fabric design.

The design of fabric components largely involves the construction of a geometric policy expressing a desired component behavior and a Finsler energy whose metric will act as that policy's priority metric. Both of these parts are defined by HD2 terms, so the question of design largely reduces to the question of how to design HD2 functions and maps.

As a reminder, a Homogeneous of Degree $k$  (HD$k$) function is a function $f(\x)$ for which $f(\alpha \x) = \alpha^k f(\x)$. In our case, our functions are typically HD0 through HD2, specifically in velocity. For instance, if $\pi(\x,\xd)$ is a geometric policy, the {\em geometric} modifier means it is HD2 in velocity, so $\pi(\x,\alpha\xd) = \alpha^2 \pi(\x,\xd)$, and usually we assuming this holds just for $\alpha\geq 0$. It places no restrictions on how the function should scale with $\x$.

Often we design HD2 terms by multiplying an HD0 term by an HD2 scaling factor. That HD0 term might be entirely independent of velocity, or it may be dependent on the direction of the velocity, but not its norm (speed). Note that in the 1D case, the velocity direction is a switch, the sign of the scalar one-dimensional velocity, switching from 1 to -1 as velocity passes from positive to negative. This discrete switch is valid with our theory (systems have to come to a rest (zero velocity) before experiencing the switch) and is used commonly in our designs, especially for one-dimensional barrier geometries and energies. 

In higher dimensions, switches depending on the norm of a velocity may not result in positive definite mass matrices. The presentation in \cite{ratliff2021geometry} allows that case by relaxing the third requirement of Definition~\ref{def:FinslerStructure} to require only invertibility, which is internally consistent with the geometric theory. But here, as a generalization of classical mechanics, we want mass matrices to remain positive definite, so building such switching terms in 1D spaces (where is naturally positive definite) and pulling them back to higher-dimensional spaces becomes an important technique for Finsler energy design.

A common pattern is to start with a scalar function $\psi(\x)$ (often similar in structure and form to standard potential functions), and either use it directly as a position only base function, or compute a base policy as $\pi_0(\x) = -\partial_\x \psi$. In this discussion we will use the base policy as an example and construct a geometric policy. This base policy is HD0 simply because it doesn't depend on velocity at all. It's often useful to depend on velocity direction $\wh{\xd} = \xd/\|\xd\|$, though. Since $\wh{\xd}$ is already independent of speed $s = \|\xd\|$, any function of velocity direction is automatically HD0 as well. That gives $\pi_0^\sigma(\x,\xd) = \sigma(\wh{\xd})\pi_0(\x)$ as a new HD0 base. Finally, any Finsler energy $\Lag_e(\x,\xd) = \frac{1}{2}\Lag_g^2(\x,\xd)$ is HD2 by definition, so scaling $\pi_0^\sigma$ by $\Lag_e$ gives us an HD2 geometry
\begin{align}
    \pi(\x,\xd) = \Lag_e(\x,\xd)\,\pi_0^\sigma(\x,\xd).
\end{align}
For instance, we might choose 
\begin{align}
    \psi(\x) &= \frac{1}{2}\|\x_0 - \x\|^2
\end{align}
as a quadratic attractor to a target $\x_0$. The negative gradient of the potential defines $\pi_0(\x) = -\partial\psi = \x_0 - \x$. Then with a slight abuse of notation, we might switch on the velocity direction using a scalar $\sigma$ of the form $\sigma(s) = 1/(1+\exp(-\beta s))$ for $\beta \in\R_+$ with $s = \vv^\tr\wh{\xd}$ for some constant vector $\vv$. The resulting HD0 policy becomes
\begin{align}
    \pi_0^\sigma(\x,\xd) 
    &= \sigma(\vv^\tr\wh{\xd}) \pi_0(\x) \\\nonumber
    &= \frac{\x_0 - \x}{1+\exp\big(-\beta \vv^\tr\wh{\xd}\big)}.
\end{align}
Finally, we can choose to scale that using a Euclidean Finsler energy $\Lag_e = \frac{1}{2}\|\xd\|^2$ to get a final geometry of the form
\begin{align}
    \pi(\x,\xd) = \frac{\frac{1}{2}\|\xd\|^2\big(\x_0 - \x\big)}{1+\exp\big(-\beta \vv^\tr\wh{\xd}\big)}
\end{align}

Note that another way to design an HD2 geometry is by designing a Finsler geometry derived from a Finsler energy. We can use one Finsler energy to derive the geometry, and a second Finsler energy to define the priority metric. Energization would then bend the geometry associated with the second (priority) energy to align with the first energy's geometry.

A common pattern for designing Finsler energies is to start with a Riemannian (kinetic) energy in 1D of the form $\mathcal{K}(x,\dot{x}) = \frac{1}{2}g(x)\dot{x}^2$ where $g(x)$ is everywhere positive and dependent on position only. We then scale it by a scaling function that depends on velocity directionality. Since we are in 1D, that normalized velocity becomes a switch on the sign of the velocity, so we can design this scaling function in two parts
\begin{align}
    \sigma(\dot{x}) = \left\{
    \begin{array}{cc}
        \sigma_-(\dot{x}) & \mbox{for $\dot{x} \leq 0$}\\
        \sigma_+(\dot{x}) & \mbox{for $\dot{x} > 0$}
    \end{array}
    \right..
\end{align}
The energy then becomes $\Lag_e(x,\dot{x}) = \frac{1}{2}\wt{g}(x,\dot{x})\dot{x}^2$ where $\wt{g}(x,\dot{x}) = \sigma(\dot{x})g(x)$, which is effectively two different metrics depending on directionality of the velocity.

Note that in the 1D case, the derived $m(x,\dot{x}) = \partial^2_{\dot{x}\dot{x}}\wt{g}$ must have the property that 
\begin{align}
\Lag_e(x,\dot{x}) = \frac{1}{2}\wt{g}(x,\dot{x})\dot{x}^2 = \frac{1}{2}m(x,\dot{x})\dot{x}^2,
\end{align}
which is only true globally when $m = \wt{g}$. This in general is not necessarily true in higher dimensions.

Finally, we describe a common pattern in designing potential functions. Potential functions define forces that push against the fabric system's mass. The mass defines the fabric priority; the stronger the mass the higher the priority the fabric takes. Thought another way, these masses define the fabric system's inertia; a given force has less effect pushing a sizable mass (high priority) than pushing a small mass (low priority). It is often important that a potential function adjust effectively to the fabric's priority. However, it can be challenging to design potential forces directly to scale in an intuitive way. 

For that reason, it's often easier to design a potential function by designing its gradient in terms of a priority metric of its own and a desired {\em acceleration}. Specifically, if $\wt{\psi}(\x)$ is a function (similar to the potential function in form) defining an acceleration policy $\pi(\x) = -\partial\wt{\psi}$, and if $\M$ is its associated priority, the resulting potential function $\psi(\x)$ would be one for which $-\partial\psi = -\M\,\partial\wt{\psi}$ (force equals mass times acceleration) if such a potential function exists. Given an expression for a vector field, it's easy to check whether it is a gradient of a potential by taking its Jacobian and verifying the Jacobian is symmetric (it's a well-known result in multivariate calculus that a vector field is the gradient of a potential if and only if its Jacobian is everywhere symmetric). 

One-dimensional vector fields are always symmetric so we can always design one-dimensional potentials using an acceleration policy and an associated priority mass. For higher-dimensional potentials, one can use the results discussed in \cite{cheng2018rmpflow} (Appendix D.4) that show that if the acceleration policy is radially symmetric and the metric is compatibly symmetric along the same radial lines, the resulting vector field defined by scaling the acceleration policy by the metric has symmetric Jacobian and is therefore the gradient of some potential. 

When designing potential functions in this way, it is necessary only to show that the potential exists. Knowing that, we can implement the system using only its gradient given by the metric-scaled acceleration policy. The resulting system then fits the theory and we have convergence and stability guarantees.

\section{Additional Experiments: Equality Constraints}

This supplementary experiment is a specific realization of constrained fabrics. In principle, constrained fabrics can be realized in many different ways, e.g., by explicitly calculating the constraint Lagrange multiplier and adding the resulting constraint force to the unconstrained fabric as discussed in Appendix \ref{apx:LemmaConstrainedGeometricFabrics}. This method would additionally need to employ Baumgarte stabilization \cite{flores2011parametric} to combat numerical drift. Other approaches may be taken as well such as solving a nonlinear program subject to a discretized momentum-impulse form of the fabric and additional equality constraints as done by Stewart and Trinkle \cite{stewart2000implicit}. There are also many other different discretization and nonlinear program schemes that could be created to solve constrained fabrics. In this particular case, we create the following discretization,

\begin{align}
    \q &= \q_k \\
    \qd &= \frac{\q_{k+1} - \q_{k-1}}{2 \Delta t} \\
    \qdd &= \frac{\q_{k+1} - 2 \q_k + \q_{k-1}}{\Delta t^2}
\end{align}
where $k$ indicates the discrete time index and $\q_{k-1}$ and $\q_k$ are known. We then minimize the following cost function,

\begin{align}
    L(\q, \qd) = \frac{1}{2} (\qdd_d - \qdd)^\tr \M(\q, \qd) (\qdd_d - \qdd) + \frac{1}{2} \lambda \C(\q_{k+1})^2
\end{align}
where $\qdd_d$ is the desired unconstrained acceleration of the fabric calculated as $\qdd_d = -\M(\q, \qd)^{-1} \f(\q, \qd)$, $\C(\q_{k+1})$ is a scalar constraint function of interest, and $\lambda$ is a weight. The first term solves for system acceleration that naturally follows the fabric while the second term is a penalty on the constraint, where $\lambda$ can be set to a large value to enforce the constraint. Altogether, this formulation implicitly solves for a new acceleration that adds a constraint force on the unconstrained fabric theoretically realizing (\ref{eqn:ConstrainedGeometricFabrics}).

We implemented the above for the geometric fabric designed in Section \ref{sec:SystemExperiments} subject to the constraint that the end-effector must lie on a plane. We repeatedly solved the program with $\Delta t = 0.04 s$ and a Gauss-Newton optimizer with four iterations per time step to forward integrate the system. The result is shown in Fig. \ref{fig:constrained_fabric}. The robot can maintain the constraint at the micrometer level and the end-effector can settle to within 1 mm of the closest point on the plane to the target.

\begin{figure}[!t]
  \centering
  \includegraphics[width=.8\linewidth]{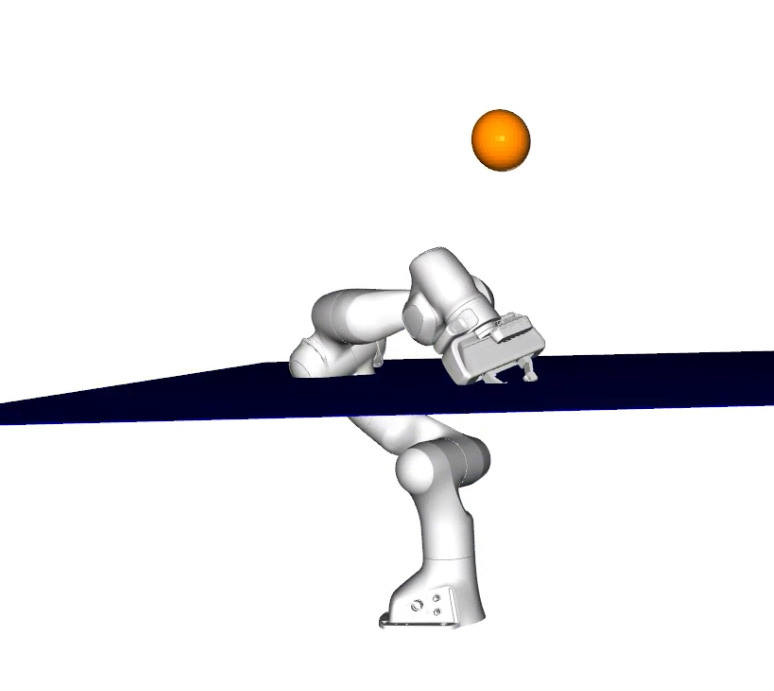}
  \caption{Constrained geometric fabric moving the robot end-effector to orange target points while the end-effector is constrained to the blue plane.}
  \label{fig:constrained_fabric}
\end{figure}

\section{End-effector Pose Control}
There are many different ways geometric fabrics can be used to control a robot, including end-effector pose control. This is a straight-forward extension to design outlined in Section \ref{sec:SystemExperiments} that attracted a point on the end-effector towards a target point. In this case, we calculate three additional task spaces that map to each axis of a orthonormal coordinate system placed at the end-effector. In each of these axis spaces, we place the very same attractor components discussed in Section \ref{subsubsec:attraction}. With this extension to the original fabric, we can exhibit full pose control over the end-effector.

\section{Additional experiments: Cubby navigation heuristics}
An important use of geometric fabrics is to heuristically bias the system without affecting convergence to local minima of the potential. This feature is especially useful for designing global navigation heuristics. A geometric fabric system as described in Subsection~\ref{sec:SystemExperiments} is subject to local minima since the target attractor is greedy. In many environments, though, simple heuristics can coarsely shape the behavior en route so that these local avoidance behaviors are sufficient for global navigation as well (similar to the heuristics described in \cite{ratliff2018rmps} Section V-D). 

Section XI of \cite{geometricFabricsForMotionArXiv2020}  describes such an experimental setup in detail. In those experiments, geometric fabrics are used to navigate between a wall of cubbies and a set of floor boxes in a standard pick and place industrial setting. The heuristic geometries encoded simple geometric fields that push the system away from the interior of any cubby and box except for the one it wanted to enter into. The final system successfully navigated between any pair of cubby and surrounding boxes even with perturbations to the environment. 

The simplicity and success of these heuristic terms suggest that learning such heuristics conditioned on environmental features over the top of an existing obstacle navigation fabric system would be straightforward, especially given automatic differentiation tools such as \cite{LiRmp2RobotLearning2021} and leveraging imitation learning methods such as \cite{xie2021imitation}. Fabrics constitute a compact encoding of behavioral components that span multiple tasks which could make it a powerful medium for learning of strongly generalizing policies. It is an area of future work to explore these learning applications in detail, especially in the context of meta learning shared nominal fabric layers. 

\end{appendices}

\end{document}